\documentclass[manuscript,nonacm]{acmart}

\usepackage{latexsym}
\usepackage{amsmath}
\usepackage{amsthm}
\usepackage{booktabs}
\usepackage{enumitem}
\usepackage{graphicx}
\usepackage{color}
\usepackage{mathpartir}
\usepackage{bussproofs}
\usepackage{tabularx}
\usepackage{makecell}
\usepackage{multirow}
\usepackage{tikz}
\usetikzlibrary{arrows.meta,positioning,calc,fit,decorations.pathreplacing}

\usepackage{tikz}
\usetikzlibrary{arrows.meta,positioning,calc,decorations.pathreplacing}

\usepackage{tikz}
\usetikzlibrary{arrows.meta,positioning,calc,decorations.pathreplacing}

\tikzset{
  flow/.style={
    -{Latex[length=2.6mm,width=1.8mm]},
    line width=.8pt,
    draw=black!55
  },
  fallback/.style={
    flow,
    dashed
  },
  mainnode/.style={
    align=center,
    font=\small
  },
  actor/.style={
    font=\small\itshape
  },
  op/.style={
    font=\scriptsize\itshape,
    fill=white,
    inner sep=1pt
  },
  chan/.style={
    font=\scriptsize,
    fill=white,
    inner sep=1pt
  },
  grouplabel/.style={
    font=\bfseries\small,
    text=green!45!black
  }
}

\newtheorem{definition}{Definition}%

\newcounter{remark}[definition]
\renewcommand{\theremark}{\thedefinition.\arabic{remark}}

\theoremstyle{remark}
\newtheorem*{remarkinner}{Remark}

\newenvironment{remark}{%
    \refstepcounter{remark}%
    \begin{remarkinner}[\theremark]%
}{%
    \end{remarkinner}%
}
\newcommand{\alt}{\;;\;}
\AtBeginDocument{%
  }

\title{Transparency as Answerability: Role-Dependent Information Flow Modeling}

\author{Mattias Brännström}
\email{mattias.brannstrom@umu.se}
\orcid{0000-0003-3113-2631}
\affiliation{%
  \institution{Umeå University}
  \city{Umeå}
  \country{Sweden}}

\author{Themis Dimitra Xanthopoulou}
\email{themis.xanthopoulou@umu.se}
\orcid{0000-0003-2914-0472}
\affiliation{%
  \institution{Umeå University}
  \city{Umeå}
  \country{Sweden}}

\author{Lili Jiang}
\email{lili.jiang@umu.se}
\orcid{0000-0002-7788-3986}
\affiliation{%
  \institution{Umeå University}
  \city{Umeå}
  \country{Sweden}}

\renewcommand{\shortauthors}{Br\"annstr\"om et al.}

\begin{abstract}
Transparency in AI-supported decision systems is widely required but poorly defined. Existing approaches treat transparency as a property of systems or components, producing characteristic failures: assessments that do not transfer across organizational boundaries, stakeholder questions that cannot be formulated within the available frame of 
analysis, and component-level findings that say nothing about system-level behavior. We propose treating transparency as a relation --- a system is transparent to a stakeholder to the degree that their legitimate questions are answerable from its structure. We formalize this using information flow networks, in which transparency questions are expressed in terms of network structure and answerability is the condition that the system settles them. The framework supports three dimensions of analysis: structural answerability from the network alone, compositional answerability across organizational boundaries for deployed systems, and role-dependent answerability given a stakeholder's accessible sub-network. We demonstrate the framework on a real pre-integration assessment of an AI recruitment tool in a public sector municipality. The assessment uses participatory design and interviews to model the decision system structure. We show how structural and role-dependent transparency questions can be formally stated and analyzed across the Provider and Deployer boundary where the AI system is composed together with the deployer's decision process.
\end{abstract}



\renewcommand\footnotetextcopyrightpermission[1]{} %
\begin{document}
\maketitle

\section{Introduction} \label{sec:introduction}
Transparency is a central requirement of AI governance, yet it is rarely given 
a precise structural account — what must be visible, to whom, and under what 
conditions remains underspecified even in the most detailed regulatory 
frameworks. The European Union's Artificial Intelligence Act is a case in 
point~\cite{EU2024AIAct}. It mandates transparency, 
accountability, and human oversight, and defines distinct roles — Provider and 
Deployer — each with formal obligations throughout the AI system 
lifecycle~\cite{EU2024AIAct}. Yet the Act specifies what transparency requires without a structural account of what satisfying those requirements means in practice — that is, without specifying how the AI system, its deployment context, and the roles that interact with it must be constituted for transparency obligations to be met.

This gap has consequences. Existing assessment frameworks treat AI systems as 
monolithic artifacts — evaluated through model-level accuracy, bias, or 
fairness metrics — without a structured account of how decisions relate to the 
broader socio-technical processes in which they operate~\cite{green_algorithmic_2020, 
selbst2019fairness}. This risks what Selbst et al. call the \textit{framing 
trap}: inferring system-level compliance from component-level 
properties, treating a model's accuracy or fairness score as if it characterizes the decision process it participates in.~\cite{selbst2019fairness}. The Provider/Deployer separation 
compounds this by imposing structural black boxes that obscure the interrelation 
among system, use-context, and affected stakeholders~\cite{VealeBorgesius2021}.

We argue that what is missing is a formalization of transparency as a \textit{relational} property --- not something a system simply has, but something it provides with respect to a specific set of legitimate questions posed by a specific stakeholder role. 

When transparency gets treated as an objective property of an AI system in isolation it leads to tokenistic disclosures not tied to any particular purpose. For transparency to be instrumental, it needs to be able to answer questions stakeholders actually need answering in a way that can be comprehended. The comprehension aspect has been argued well by others \cite{paez2019pragmatic,miller2017inmates,lipton2018mythos} --- this perspective is important but, by itself insufficient. 

For transparency to be instrumental, the information it provides must therefore be:
\begin{enumerate}
    \item \textbf{Structurally grounded}, because not all systems can answer all questions. If information is not present or not used in the required way, then it does not matter how pedagogical the answer is.
    \item \textbf{Meet particular information needs}, because stakeholders have obligations, rights or ethical concerns upon them, and they need this information in order to meet these requirements. 
    \item \textbf{Socio-technical in scope}, because the relevant questions asked are socio-technical. 'Why was the candidate rejected?' is not a question about an algorithm, it is a question about a process that involves candidates, acceptance criteria, competition and decisions which affect someone's life. 
\end{enumerate}

By \textit{role} we therefore mean stakeholders with particular responsibilities and accountability relations to the system --- candidates, recruiters, regulatory auditors. The \emph{information demands} of the stakeholders of this role, what they need to know to meet their responsibilities, we term the \textit{legitimate questions} of this role.

Building on the Information Flow Methodology (IFM)~\cite{brannstrom2023impact,ifm2025guide,brannstrom2025clear}, which models decision systems as networks of information sites and transformation channels, this paper develops a theoretical and formal account of answerability over such networks.

With this in place we can understand transparency as the relationship between what information a role needs and what answers a system can provide.

We ground the formalism in a real-world pre-integration assessment of a high-risk AI recruitment tool, applied in collaboration with an external auditor at a public-sector municipality. 

The paper addresses three research questions:
\begin{enumerate}
    \item Can transparency be formally characterised as a relation between a 
    system and a set of questions rather than a property of a system alone?

    \item When AI systems are composed within an organizational process, what compositional structure is required for transparency questions that span organisational and technical boundaries to be answerable?

    \item How does role-dependent access determine which transparency questions are answerable, and what are the structural conditions under which a system is transparent to a given role?
\end{enumerate}

The remainder of this paper is structured as follows. Section~2 introduces the regulatory background and motivates the need for a structural account of 
transparency. Section~3 develops the formal framework. Section~4 presents the use case. Sections~5--7 apply the framework to structural, compositional, and 
role-dependent answerability respectively. Section~8 discusses implications and limitations, and Section~9 concludes.

\section{Background} \label{sec:background}

The structural account of transparency developed in this paper responds to a specific regulatory and scholarly context. A central innovation of the EU AI Act~\cite{EU2024AIAct} is its distribution of obligations for high-risk systems across two operational \textit{actors}: the \textit{Provider} (developer or vendor) and the \textit{Deployer} (user organization). This separation is a substantive structural choice --- transparency, oversight, and accountability obligations attach differently depending on where an actor sits relative to the system --- and it reflects real asymmetries in access, expertise, and control across the AI lifecycle.

The regulation and the practical landscape of multi-actor, multi-component contexts together create a concrete analytical challenge for AI transparency. Once obligations are divided across organizational actors, the questions stakeholders can ask and the information available to answer them likewise divide, and transparency must be analysed across the boundary rather than within either side of it.

This division is also rarely clean-cut in practice: systems are adapted after delivery, organisations act as both developer and user, and documentation is often produced under one actor's perspective and read under another's~\cite{VealeBorgesius2021, green_algorithmic_2020}. Tools for AI governance, like the European Commission's ALTAI, the Dutch FRAIA template,\footnote{https://www.government.nl/documents/reports/2021/07/31/impact-assessment-fundamental-rights-and-algorithms} and the Swedish DIGG trust model tend to adopt a single-actor (typically focused on the developer) or linear-pipeline perspective, and so address one side of the boundary at a time rather than the transparency relations that span it (we return to this in Section \ref{sec:discussion}).

Two bodies of work engage the different sides of this gap. The first situates accountability within socio-technical context. Metcalf et al. argue that accountability must be tied to actual sites of impact rather than abstracted through technical compliance~\cite{metcalf2021algorithmic}; Weinberg shows that fairness metrics cannot substitute for attention to concrete social consequences~\cite{weinberg2022rethinking}; Binns locates stakeholder impact in deployment context rather than model behaviour~\cite{binns2018fairness}; Green and Viljoen critique the assumption that algorithms can be evaluated apart from context~\cite{green_algorithmic_2020}; Selbst et al. name the underlying error the \emph{framing trap} --- inferring system-level properties from component-level evaluation~\cite{selbst2019fairness}; and Ananny and Crawford argue that transparency without specification of audience and question becomes another form of the same abstraction~\cite{ananny2018seeing}. These critiques establish that transparency cannot be inferred from properties of an AI component alone. It must be evaluated relative to the decision process as deployed: what information exists, how it is transformed and preserved, who can access it, and how these relations extend across technical and organizational boundaries. 

At the same time, XAI has developed a broad range of methods for extracting and communicating information about AI models and their behaviour, including feature-attribution, counterfactual, surrogate-model, and example-based explanations~\cite{arrieta2020explainable}. These approaches provide increasingly sophisticated means of producing answers about particular aspects of technical systems, but do not by themselves determine which questions a given system is structurally capable of answering. This limitation becomes especially important once the object of analysis extends beyond a model to a composed or socio-technical decision process, where relevant information may be distributed across multiple components, transformations, and actors. The resulting gap is therefore between information demand and system structure: we lack a general formal account connecting the questions asked of a system to the information that its structure can support.

We address this gap by treating transparency as answerability. A question specifies an information requirement; the system determines whether the required information exists, how it is transformed and preserved, and whether the question can be settled from its structure. Where questions are associated with stakeholder roles, this relation becomes role-dependent; where systems cross organizational and technical boundaries, it becomes compositional.

Transparency demands are naturally expressed as questions, explicitly asked or just wondered: a candidate asks why they were rejected, a recruiter wonders how scores were produced, an auditor asks if all candidates are given equal treatment. Questions have become an established unit of explanatory demand in XAI~\cite{miller2019explanation,liao2020questioning}, and we adopt the same starting point for transparency more broadly. This premise --- that explanations answer questions --- has a longer history in philosophy of science, developed with different emphases by van Fraassen's pragmatic theory of explanation~\cite{vanFraassen1980} among others; there, though, the concern is what makes an answer explanatorily \textit{good}, selecting among true candidate answers via a contrast class and an interest-relative relevance relation. Our concern is prior to that: not which answer is best, but whether one settling the question exists in the system's structure at all. The formal notion of settlement, and the semantic theory of questions itself, we take from inquisitive semantics~\cite{ciardelli2018inquisitive}, which characterizes questions in terms of the information states that settle them.

Making this precise requires formal tools on both sides: for the questions themselves, and for the systems they are asked of. Not all questions can be answered by any structure, and questions themselves, via their presuppositions, can contain structural claims. Different transparency questions may also concern different depths of system description, and thus require different kinds of structure and access to be answered.

Thus, what is needed is an account of the conditions under which a system makes the information behind a question accessible to a given role. On the system side, information flow networks make the structural features that determine this --- what information exists, how it flows, what a role can access --- explicit. Inquisitive semantics has not previously been applied to socio-technical transparency, and information flow networks have not previously been used as the semantic substrate for answerability.

We conclude this section with key related work and by clarifying how transparency, answerability, and explainability relate across the fields this work engages.

\subsection{Related Work}

Several frameworks are closely related to the present account on the institutional and accountability side. Raji et al.~\cite{raji2020closing} and Cobbe et al.~\cite{cobbe2021reviewable} develop internal auditing and reviewability models grounded in documentation, organizational processes, and differentiated responsibilities. Cobbe and Singh~\cite{cobbe2023supply} trace accountability across algorithmic supply chains. Novelli et al. develop scenario-based and proportional approaches to risk assessment under the AI Act~\cite{novelli2023taking, novelli2024scenario}, and OMNI~\cite{dignum2005omni} formally represents agent roles and normative constraints within multi-agent systems. These approaches provide important structure for roles, responsibilities, documentation, and risk across socio-technical systems. They do not, however, provide a common formal representation of how information is generated, transformed, and made accessible across both technical components and organizational processes. It is this relation between stakeholder information demands and the information-flow structure of the resulting socio-technical system that the present framework addresses.

Creel distinguishes three dissociable forms of computational transparency: functional transparency, concerning knowledge of the algorithmic functioning of the whole; structural transparency, concerning how that functioning is realized in code; and run transparency, concerning the program as executed in a particular instance~\cite{creel2020transparency}. Which form of transparency is useful is task-dependent. Our distinction between structural and trace-level answerability (Sections~\ref{subsec:structural-answerability}--\ref{subsec:trace-questions}) partially overlaps with this taxonomy but serves a different purpose: it asks whether a stakeholder question can be settled from the modeled process structure or instead requires information about a particular implementation or trace. Our use of \textit{structural} is therefore not Creel's code-level sense and, for computational components, often lies closer to her functional transparency. Creel's account also concerns epistemic access and understanding, whereas the present framework deliberately brackets the recipient's capacity to understand an available answer; we return to this as a question of effective transparency in Section~\ref{sec:discussion}.

Two formal-methods traditions share machinery with the present framework and deserve explicit comparison. Refinement calculus~\cite{back1998refinement} formalizes a stepwise relation between specifications and implementations, with refinement preserving correctness across levels of abstraction. Information flow security in the Goguen--Meseguer tradition~\cite{goguen1982security, sabelfeld2003language} characterizes noninterference as a relation between what is observable at different security levels, and has developed refinement-like notions for comparing how information propagates through a system. Both use refinement orderings and level- or role-dependent observability, and there is some structural similarity to our framework. The difference lies in what the refinement is \textit{for}: in refinement calculus, toward an executable program satisfying a specification; in information flow security, toward guarantees about what a given observer can or cannot learn; in our framework, toward a representation in which a stakeholder's questions are answerable. Our refinement relation is descriptive rather than prescriptive, and the property it supports --- answerability of transparency questions --- has no direct counterpart in either tradition.

\subsection{Transparency, Answerability, and Explainability}

Transparency, answerability, and explainability are terms used across legal, policy, and technical fields with partially overlapping but distinct meanings. This has led to considerable confusion about what AI systems are actually required to provide and to whom. We briefly clarify how we use these concepts within the present framework and how this use relates to other fields and legal obligations. We will do so by drawing a demand and supply sketch of a transparency relationship.

The demand side concerns what someone requires or needs to know. This requirement or need might have many different sources and is strongly dependent on the particular use case and situation. In this paper we generalize this with the concept of a \textit{role}, representing a group or type of stakeholder with a particular need for information. The clearest view of the demand side comes from \textit{legal explainability}, which specifies what affected persons are entitled to know about decisions that concern them.
This does not only involve those who are affected by a decision but just as much stakeholders with rights or obligations relating to the automated process, such as users, managers, auditors and support. In this paper we term these information demands the \textit{legitimate questions} of a stakeholder role, denoted $Q$, or $Q_R$ when indexed to a role. 

The supply side of this relationship concerns what information is available to satisfy the needs of the demand side. Here we find disclosure rules, reviewability frameworks and eXplainable AI (XAI). \textit{XAI explainability} concerns the extraction and delivery of explanations from an AI model (e.g a neural network or other AI artifact). Post-hoc methods --- saliency maps, counterfactual generators, surrogate models --- provide explanations aimed at informing some audience about the models inner workings or properties~\cite{arrieta2020explainable}. In some cases such explanations can be pragmatically satisfying without being structurally grounded in the system's actual decision process~\cite{paez2025axe}. \textit{Answerability}, in the sense developed in this paper, stands prior to both extraction and delivery of explanation: it asks what information capable of settling questions exists in the system's structure at all, before any extraction or delivery is attempted. A truthful explanation is only possible where answerability holds --- XAI methods applied to a system that does not structurally support a question can produce explanations that are at best approximations and at worst misleading.

The transparency framework developed here occupies a bridge position: it formalizes the structural precondition when particular information demands (demands side, from stakeholders) can be paired with answerability (supplier side, system) of that same information. This bridge represents both a realization of demands, and a conditional on what transparency means, relative to such demands. As Kaminski argues: \textit{``the substance of other underlying legal rights often
determines transparency’s substance."} If one has a right of correction or against discrimination, one needs to see errors and factors used, respectively. Without this correspondence, the rights are effectively void~\cite{kaminski2021right}. The framework directly lends to a reading of the requirement that explanations be meaningful (Art. 86 AI Act, and cf. Art. 15(1)(h) GDPR): an explanation may be considered meaningful for a role when it settles the questions that role legitimately needs answered. Like the rest of the framework this is a structural condition and not a doctrinal interpretation; the legitimate questions of roles can be constructed from (e.g.) legal interpretation, but are answered by structure. 

A system is thus transparent to a role when it is in principle able to provide the information this role needs. This, in principle, is of some importance to note since it explicitly ignores the delivery and capability of appreciation of the information sought. The requirement that information is to be understandable to the recipient is an important part of transparency of AI. The framework in this paper does not however engage with this question for two reasons. Firstly it is because the answerability of the system is a precondition to whether the demand can be satisfied no matter how pedagogical the delivery is. It is, in our opinion, of high importance that transparency and explanations are in fact true of the underlying process --- if this is not the case, the instrumental benefits of transparency e.g., for effective use, contestability, rights, fairness and safety --- are not realized. Secondly, in this framework the information in question is scoped by the demand, the legitimate questions of a role. Because these questions arise from the role's own rights, obligations, and practical responsibilities, their corresponding answers are already selected from a substantially narrower region of the system's information space that is comparatively meaningful and comprehensible to that role. This does not guarantee that every answer or representation will be understood, but it substantially constrains the remaining problem of presentation and comprehension. 

Note also that the relation between roles and questions developed here is demand-based: the legitimate questions of a role derive from its rights, obligations, and practical responsibilities, and they do not disappear when a stakeholder happens to already know the answer — it still makes sense to say a candidate must be able to answer "Was I hired?" even where that candidate knows the outcome.

The following Section ~\ref{sec:3_transparency_answerability} formalizes these structural conditions using information flow networks as the substrate that anchors the relation between the demands of a role and what the system can supply. 

\section{Transparency as Answerability}
\label{sec:3_transparency_answerability}
This section develops the formal machinery that lets transparency questions be stated and evaluated precisely. 

The construction has three layers. Information Networks (Definitions \ref{def:network}–-\ref{def:refinement}) are the structural objects: they represent decision processes as information flows and admit a refinement order capturing what it means for one network to be more specific than another. This is needed for us to talk about both what a particular structure captures and if two particular structures are agreeing or conflicting with each other. 

Propositions over networks (Definition \ref{def:propositions}) are the things that can be asked and settled; we strongly borrow this layer from \textit{inquisitive semantics}\cite{ciardelli2018inquisitive}, but with information networks playing the role possible worlds play there. This is what allows us to treat information networks, and what we attach to them and describe by them as propositions and questions.

Having established the base which allows us to talk about network structure, \textit{descriptions} (Definitions \ref{def:descriptions}–\ref{def:described-refinement}) attach substantive content to the structural scaffolding, so that site values and channel transformations — not just connectivity — can enter into what is proposed, asked and answered. 
The payoff is Proposition \ref{prop:answerability}, where answerability becomes a precise relation between a system and a question: conformance, partial settling, and full settling each have exact conditions on the underlying propositions and this information network structure.

\begin{definition}[Information Network] \label{def:network}

An \emph{information network} is a tuple $\mathbf{N} = \langle \mathbf{S}, 
\mathbf{C}, \geq \rangle$ where:
\begin{itemize}
    \item $\mathbf{S}$ is a finite set of \emph{sites}, representing states or repositories of information;
    \item $\mathbf{C}$ is a finite set of \emph{channels}, where each channel $C \in \mathbf{C}$ is a directed hyperedge connecting a non-empty set of input sites to a non-empty set of output sites;
    \item $\geq$ is a partial order on $\mathbf{S}$ such that for every channel $C \in \mathbf{C}$, each input site is greater than each output site.
\end{itemize}
The maximal elements of $\mathbf{S}$ under $\geq$ are the \emph{input sites} 
$\mathit{In}(\mathbf{N})$ and the minimal elements are the \emph{output sites} 
$\mathit{Out}(\mathbf{N})$ of $\mathbf{N}$.
\end{definition}

\begin{definition}[Refinement]
\label{def:refinement}
Let $\mathbf{N}_1$ and $\mathbf{N}_2$ be information networks. $\mathbf{N}_1$ is a \emph{refinement} of $\mathbf{N}_2$, written $\mathbf{N}_1 \preceq \mathbf{N}_2$, if $\mathbf{N}_1$ can be obtained from $\mathbf{N}_2$ by any finite combination of the following rules of refinement:

\textsc{Site Split}: \quad
$\phi \to A \to \psi \;\vdash\; \phi \to A{:}[A_1, \ldots, A_n] \to \psi$

\textsc{Ch. Intro}: \quad
$\phi \to A{:}[A_1 \sqcup A_2] \to \psi,\; A_1 \| A_2 \;\vdash\; \phi \to A_1 \to A_2 \to \psi$

\textsc{Ch. Split}: \quad
$\phi \to A \to B \to \psi \;\vdash\; \phi \to A \to C \to B \to \psi$

\textsc{Ch. Sharp. (disconnect)}: \quad
$\phi \to [A_1 \sqcup A_2] \to B \to \psi \;\vdash\; \phi \to A_1 \to B \to \psi,\; A_2$

\textsc{Ch. Sharp. (distribute)}: \quad
$\phi \to [A_1 \sqcup A_2] \to [B_1 \sqcup B_2] \to \psi \;\vdash\; \phi \to A_1 \to B_1 \to \psi,\; A_2 \to B_2$

\textsc{Out Sharp.}: \quad
$\phi \to A \to \psi \;\vdash\; \phi \to A \to \psi,\; A \to X$

\textsc{Situate}: \quad
$A \in \mathit{In}(\mathbf{N}),\; \phi \to A \to \psi \;\vdash\; X \to A \to \psi$.

Here $\phi$ and $\psi$ denote possibly empty upstream and downstream network context; site X is fresh; $A_1 \| A_2$ denotes incomparability under site pre-order $\geq$ and $\sqcup$ denotes disjoint decomposition of sub-sites.

We call the \emph{language} $L(N)$ of $\mathbf{N}$ all networks which can be obtained from $N$ through refinement.

\end{definition}

\begin{proposition}
\label{prop:refinement-order}
The refinement relation $\preceq$ is a partial order on information networks. $N' \preceq N$ iff $N' \in L(N)$ and
$L(\mathbf{N}) \subseteq L(\mathbf{Q}) \iff \mathbf{N} \preceq \mathbf{Q}$.
Two networks are \emph{incomparable} under $\preceq$ if neither refines 
the other. The least refined network is the single site $S$. The refinements of $S$, $L(S)$, contain all networks, which we term $\mathbb{IF}$.
\end{proposition}

\begin{proof}
\emph{Reflexivity} and \emph{transitivity} follow directly: zero rule 
applications yield $\mathbf{N} \preceq \mathbf{N}$, and derivation sequences compose.
For \emph{Antisymmetry}: every rule introduces fresh sites or ordering and none removes them, so no rule application is reversible. If $\mathbf{N}_1 \preceq \mathbf{N}_2$ and $\mathbf{N}_2 \preceq \mathbf{N}_1$ then the rule sequence deriving each from the other must be empty, giving $\mathbf{N}_1 = \mathbf{N}_2$.

$\mathbf{N}' \preceq \mathbf{N} \iff \mathbf{N}' \in L(\mathbf{N})$ follows directly from the definition of $L(\mathbf{N})$ as the set of all refinements of $\mathbf{N}$.

If $L(\mathbf{N}) \subseteq L(\mathbf{Q})$ then $\mathbf{N} \in L(\mathbf{Q})$ since $\mathbf{N} \in L(\mathbf{N})$ by reflexivity, giving $\mathbf{N} \preceq \mathbf{Q}$. Conversely, if $\mathbf{N} \preceq \mathbf{Q}$ then for any $\mathbf{N}' \in L(\mathbf{N})$, $\mathbf{N}' \preceq \mathbf{N} \preceq \mathbf{Q}$ by transitivity, so $\mathbf{N}' \in L(\mathbf{Q})$.
\end{proof}

\begin{definition}[Information States and Propositions]
\label{def:propositions}
We describe propositions directly analogous to Inquisitive Semantics~\cite{ciardelli2018inquisitive}, where information states are sets of possible worlds and propositions are downward-closed sets of information states. Here networks take the role of possible worlds.

An \emph{information state} is a set of information networks $K \subseteq \mathbb{IF}$, representing a state of knowledge about which network describes the actual system — the networks consistent with what is known.

A \emph{proposition} is a non-empty downward-closed set of information states:
\[
\mathcal{P} \subseteq 2^{\mathbb{IF}} \text{ such that } K \in \mathcal{P} \text{ and } K' \subseteq K \implies K' \in \mathcal{P}
\]
Information states are ordered by inclusion: $K_1 \preceq K_2$ iff $K_1 \subseteq K_2$. A smaller information state is more refined — it admits fewer possible networks and represents more specific knowledge.

The \emph{language} of a network $L(\mathbf{N}) = \{ \mathbf{N}' \mid \mathbf{N}' \preceq \mathbf{N} \}$ is an information state — the set of 
all networks at least as specific as $\mathbf{N}$. The proposition generated by $\mathbf{N}$ is $\mathcal{L}(\mathbf{N}) = \{ K \mid K 
\subseteq L(\mathbf{N}) \}$ — the downward closure of $\{L(\mathbf{N})\}$ under information state inclusion.

For a proposition $\mathcal{P}$:
\begin{itemize}
    \item \emph{Alternatives} are the maximal elements of $\mathcal{P}$ under $\subseteq$ — the most specific information states in $\mathcal{P}$ that are not subsets of each other. Alternatives of a proposition are denoted by $\mathbf{?}\mathcal{P}$.
    \item The \emph{informative content} $\mathbf{!}\mathcal{P} = \bigcup \mathbf{?}\mathcal{P}$ — the union of all alternatives, representing what is common to all of them.
    \item $\mathcal{P}$ is \emph{inquisitive} if $|\mathbf{?}\mathcal{P}| > 1$ — multiple alternatives remain open.
    \item $\mathcal{P}$ is \emph{settled} if $|\mathbf{?}\mathcal{P}| = 1$ — a unique most specific information state.
    \item $\mathcal{P}$ is \emph{informative} if $\mathbf{!}\mathcal{P} \neq \mathbb{IF}$ — it excludes some networks.
\end{itemize}
\end{definition}

\begin{proposition}
\label{prop:lattice}
The set of all propositions over information networks, ordered by inclusion, forms a complete distributive lattice.

The top proposition $\top$ admits $\mathbb{IF}$, all networks. 
The bottom proposition $\bot$ admits $\{\emptyset\}$, none.
\end{proposition}

\begin{proof}
\emph{Bounds}: $\mathcal{T} = 2^{\mathbb{IF}}$ is trivially downward-closed and serves as $\top$. $\mathcal{B} = \{\emptyset\}$ is downward-closed since $\emptyset$ has no proper subsets, and serves as $\bot$.

\emph{Pairwise meets and joins}: For any $\mathcal{P}, \mathcal{Q}$, their intersection $\mathcal{P} \cap \mathcal{Q}$ and union $\mathcal{P} \cup \mathcal{Q}$ are both downward-closed: if $K \in \mathcal{P} \cap \mathcal{Q}$ and $K' \subseteq K$ then $K' \in \mathcal{P}$ and $K' \in \mathcal{Q}$ by downward closure of each, and similarly for union.

\emph{Completeness}: For any $\mathcal{F}' \subseteq \mathcal{F}$, both $\bigcap \mathcal{F}'$ and $\bigcup \mathcal{F}'$ are downward-closed: if $K \in \bigcup \mathcal{F}'$ then $K \in \mathcal{P}$ for some $\mathcal{P} \in \mathcal{F}'$, and any $K' \subseteq K$ is in $\mathcal{P}$ by downward closure, hence in $\bigcup \mathcal{F}'$. The argument for $\bigcap \mathcal{F}'$ is analogous.

\emph{Distributivity}: $\mathcal{P} \cap (\mathcal{Q} \cup \mathcal{R}) = (\mathcal{P} \cap \mathcal{Q}) \cup (\mathcal{P} \cap \mathcal{R})$ for all sets, and hence for propositions.
\end{proof}

\begin{definition}[Alternative Notation]
\label{def:alternatives}
Alternative notation provides a compact way of writing inquisitive propositions directly. We distinguish two forms:

\begin{itemize}
    \item \emph{Existence inquiry.} For a sub-network $\mathbf{M}$, the expression $\mathbf{M}?$ denotes the proposition asking whether $\mathbf{M}$ is present:
    \[
    \mathbf{M}? := \{\,L(\mathbf{N}\text{ with }\mathbf{M}) \alt L(\mathbf{N}\text{ without }\mathbf{M})\,\}.
    \]

    \item \emph{Contents inquiry.} For any object $X$ whose admissible alternatives determine information states $K_1,\ldots,K_n$, the expression $\mathbf{?}X$ denotes the proposition whose alternatives are exactly those information states:
    \[
    \mathbf{?}X := \{\,K_1 \alt \cdots \alt K_n\,\}.
    \]
\end{itemize}

More generally, for explicitly given alternative sub-structures $\mathbf{M}_1,\mathbf{M}_2$, we write
\[
\mathbf{?}\{\mathbf{M}_1 \alt \mathbf{M}_2\} := \{\,L(\mathbf{N}\text{ with }\mathbf{M}_1) \alt L(\mathbf{N}\text{ with }\mathbf{M}_2)\,\}.
\]

Here $\{K_1 \alt \cdots \alt K_n\}$ denotes the proposition whose alternatives are exactly the information states $K_1,\ldots,K_n$.

These expressions denote propositions, not networks. For example, 

$A \to \mathbf{?}B \to C$ asks what intermediate content $B$ takes between $A$ and $C$, 

$\mathbf{?}d(S)$ asks which admissible description $d(S)$ has,
whereas 

$(A \to B \to C)?$ asks whether that sub-network is present at all.
\end{definition}

\begin{definition}[Descriptions]
\label{def:descriptions}
A \emph{description} attaches informative content to the sites and channels of a bare network $\mathbf{N}$. Each site $S \in \mathbf{S}$ is assigned a \emph{site description} $d(S)$ — a proposition over the possible values $S$ can take. Each channel $C \in \mathbf{C}$ is assigned a \emph{channel description} $d(C)$ — a characterisation of the transformation $C$ performs.

A site description $d(S)$ is:
\begin{itemize}
    \item \emph{settled} if it contains a single value — $S$ takes one determinate value;
    \item \emph{inquisitive} if it contains multiple alternatives — the value of $S$ is open.
\end{itemize}
A \emph{described network} is a bare network $\mathbf{N}$ together with descriptions for all its sites and channels. The descriptions do not alter the structure of $\mathbf{N}$ — sites, channels, and the partial order $\geq$ remain as in Definition~\ref{def:network}.

\end{definition}

\begin{remark}
The kind of content carried by site descriptions may vary: qualitative 
descriptions of what information a site holds, specific values in a run, 
probability distributions over alternatives, or transformation functions in 
an implementation. The definition accommodates all of these uniformly as 
propositions over possible values.
\end{remark}

\begin{definition}[Refinement of Described Networks]
\label{def:described-refinement}
Let $\mathbf{N}_1$ and $\mathbf{N}_2$ be described networks. $\mathbf{N}_1 \preceq \mathbf{N}_2$ if $\mathbf{N}_1$ can be obtained from $\mathbf{N}_2$ by any finite combination of the structural refinement rules (Definition~\ref{def:refinement}) together with the following descriptive refinement rules:

\textsc{Desc. Refine}: \quad
$d'(A) \preceq d(A),\; \phi \to d(A) \to \psi \;\vdash\; \phi \to d'(A) \to \psi$,
where $\epsilon \preceq d(A)$ for any description $d(A)$.

\textsc{Site Split-D}: \quad
$d(A_1), \ldots, d(A_n) \lesssim d(A),\; \phi \to d(A) \to \psi \;\vdash\; \phi \to d(A){:}[d(A_1), \ldots, d(A_n)] \to \psi$,
where $\lesssim$ denotes that the component descriptions jointly constitute a valid decomposition of the parent description, relative to the description language in use. Further formalization of $\lesssim$ is left to further work but might not be fully specifiable for linguistic descriptions.
\end{definition}

\begin{proposition}
\label{prop:described-order}
Combined structural and descriptive refinement is a partial order on described networks.
\end{proposition}

\begin{proof}
Reflexivity and transitivity follow as in Proposition~\ref{prop:refinement-order}. For antisymmetry, \textsc{Desc. Refine} strictly specializes or introduces descriptions and \textsc{Site Split-D} ensure that component sites capture or refine the original site description — neither is reversible by any rule. The argument of Proposition~\ref{prop:refinement-order} therefore applies directly.
\end{proof}

Having built propositions and questions, we can now state precisely what it means for a system to answer one. The following proposition gives three conditions --- conformance, partial settling, and full settling --- that determine how a system proposition $\mathcal{P}$ relates to a question $\mathcal{Q}$. This is the root result the rest of the paper builds on. %

\begin{proposition}[Answerability]
\label{prop:answerability}
Let $\mathcal{Q}$ be an inquisitive proposition with informative content 
$!\mathcal{Q}$ and inquisitive content $\mathbf{?}\mathcal{Q} = \{L(\mathbf{N}_1), 
\ldots, L(\mathbf{N}_k)\}$, and let $\mathcal{P}$ be a proposition 
representing a system. Then:

\begin{enumerate}
    \item \emph{Conformance} (necessary for any settling): 
    $\mathcal{P} \subseteq !\mathcal{Q}$ — $\mathcal{P}$ must conform 
    to the presupposition of $\mathcal{Q}$.

    \item \emph{Partial settling} (sufficient for at least partial 
    settling):
    \[
    \mathcal{P} \subset !\mathcal{Q} \quad \text{and} \quad 
    \mathcal{P} \not\subseteq \text{meet}(\mathcal{Q})
    \]
    Any such $\mathcal{P}$ excludes at least one alternative in 
    $\mathbf{?}\mathcal{Q}$, favouring some alternatives over others. When 
    $\text{meet}(\mathcal{Q}) = \bot$ — the alternatives are mutually 
    incompatible — the second condition is trivially satisfied for any 
    non-$\bot$ $\mathcal{P}$, and strict refinement of $!\mathcal{Q}$ 
    alone suffices.

    \item \emph{Full settling} (necessary and sufficient for full 
    settling):
    \[
    \exists!\, L(\mathbf{N}_i) \in \mathbf{?}\mathcal{Q} : \mathcal{P} \subseteq 
    L(\mathbf{N}_i)
    \]
    $\mathcal{P}$ is contained within the language of exactly one 
    alternative — selecting it and excluding all others.
\end{enumerate}
\end{proposition}

\begin{proof}
\emph{Conformance}: If $\mathcal{P} \not\subseteq !\mathcal{Q}$ then 
$\mathcal{P}$ contains a network outside $!\mathcal{Q}$, meaning it 
does not conform to the shared structure of $\mathcal{Q}$'s alternatives. 
Then $\mathcal{P} \cap \mathcal{Q} = \emptyset$ and no settling is possible.

\emph{Partial settling}: If $\mathcal{P} \subset !\mathcal{Q}$ and 
$\mathcal{P} \not\subseteq \text{meet}\mathcal{Q}$ then $\mathcal{P}$ 
has committed to structure beyond $!\mathcal{Q}$ that is not common to all 
alternatives. Therefore $\mathcal{P} \cap L(\mathbf{N}_j) = \emptyset$ for 
at least one $j$, excluding that alternative. 
When $\text{meet}(\mathcal{Q}) = \bot$, the second condition holds for any 
non-$\bot$ $\mathcal{P}$.

\emph{Full settling}: If $\mathcal{P} \subseteq L(\mathbf{N}_i)$ then 
$\mathcal{P} \cap L(\mathbf{N}_j) = \emptyset$ for all $j \neq i$ since 
$\mathbf{N}_i$ and $\mathbf{N}_j$ are distinct maximal elements — no 
network refines into both. Conversely, if $\mathcal{P}$ fully settles 
$\mathcal{Q}$ then exactly one $\mathcal{P} \cap L(\mathbf{N}_i) \neq 
\emptyset$, meaning $\mathcal{P} \subseteq L(\mathbf{N}_i)$.
\end{proof}

\begin{definition}[Projection]
Let $\mathbf{N}$ be an information network and $c \in \mathbf{C}$ a channel. 
The \emph{projection} $\mathbf{N}|_c$ is the sub-network of $\mathbf{N}$ 
induced by $\{s \in \mathbf{S} \mid \mathit{In}(c) \geq s \geq \mathit{Out}(c)\}$.
\end{definition}

\begin{definition}[Reachability]
Let $\mathbf{N}$ be an information network. The \emph{downstream} of a site 
$s \in \mathbf{S}$ is:
\[
\downarrow(s) = \{s' \in \mathbf{S} \mid s \geq s'\}
\]
The \emph{upstream} of a site $s \in \mathbf{S}$ is:
\[
\uparrow(s) = \{s' \in \mathbf{S} \mid s' \geq s\}
\]
A site $a$ \emph{reaches} a site $b$ if $b \in \downarrow(a)$. 
The \emph{path} from $a$ to $b$ is $\pi(a,b) = \downarrow(a) \cap \uparrow(b)$, 
which is non-empty if and only if $a$ reaches $b$.
\end{definition}

\begin{remark}
$\pi(a,b) \neq \emptyset$ if and only if $a$ reaches $b$. If $a$ reaches 
$b$ then $b \in \downarrow(a)$ and $b \in \uparrow(b)$ by reflexivity of 
$\geq$, so $b \in \pi(a,b)$. Conversely, if $s' \in \pi(a,b)$ then 
$a \geq s' \geq b$ and transitivity of $\geq$ gives $a \geq b$.
\end{remark}

With the base formal machinery in place; the remaining sections apply and extend it to structural (Section~\ref{sec:structural}), compositional (Section~\ref{sec:compositional}), and role-dependent (Section~\ref{sec:role-dep}) answerability.

\section{The Use Case}
\label{sec:use-case}

This paper draws on a pre-integration assessment of a commercial AI-supported recruitment tool under consideration for adoption by a public sector municipality. 
The municipality handles a broad range of recruitment needs --- from senior management 
positions to high-volume temporary staffing --- and operates under established guidelines for competence-based and equal-treatment recruitment. The AI tool is a high-risk system under the EU AI Act, a classification the Provider explicitly 
accepts.

The case material was developed through semi-structured stakeholder interviews and document analysis, conducted as part of a pre-integration assessment commissioned by the Deployer organization. A participatory modeling and elicitation approach was employed throughout. Interviews were held with the Deployment organization including recruiters, hiring managers and, technical staff, as well as with representatives of the Provider, amounting total more than 10 hours of 1:1 interviews. These interviews covered job role, system functionality, organizational practices, anticipated integration scenarios and responsible AI practices. In addition to the interviews documentation on the Deployer's internal policies and guidelines as well as the Providers compliance documentation was examined and considered in the assessment.

Information flow models were constructed iteratively from interview material and validated with stakeholders against their understanding of the systems and processes involved. A full account of the IFM modeling methodology is available in~\cite{ifm2025guide}.%
The pre-integration perspective --- that no integration had yet taken place at the time of the assessment --- is analytically valuable: it makes the Provider/Deployer boundary visible as a structural fact rather than obscuring it within a single deployed system.

Both the Deployer and the Provider demonstrated a strong commitment to AI Act compliance and responsible AI use, which facilitated a more thorough and in-depth case study.

In the following two sections we present both the Deployer system (the municipal recruitment process before any tool AI integration) and the Provider system (the AI recruitment support tool) as information flow models. This will allow us to see how they operate, reason over them and later in Section \ref{sec:compositional} compose them.  

\subsection{The Deployer Process}

In this section we describe the Deployer's recruitment process, conceptually and as an information flow model.

\begin{table*}[t]
\centering
\resizebox{0.95\textwidth}{!}{%
\begin{tabular}{cllll}
\toprule
\textbf{\#} & \textbf{Channel} & \textbf{Performed by} & \textbf{Operation} & \textbf{Stage} \\
\midrule
d1 & $\mathit{D_0} 
    \xrightarrow{\text{Application}} \mathit{CV}$ 
    & Candidate & CV submission & Sourcing \\
d2 & $[\mathit{RP}], \mathit{CV} 
    \xrightarrow{\text{Initial scoring}} \mathit{Score_R}$ 
    & Recruiter & 4-grade candidate scoring & Assessment \\
d3 & $[\mathit{RP}], \mathit{CV} 
    \xrightarrow{\text{Initial scoring}} \mathit{Score_{HM}}$ 
    & Hiring Manager & 4-grade candidate scoring & Assessment \\
d4 & $[\mathit{RP}], \mathit{D_0} \times \mathit{Score_R}, \mathit{Score_{HM}}
    \xrightarrow{\text{Selection meeting}} \mathit{D_1}$ 
    & Recruiter, Hiring Manager & Collaborative candidate selection & Selection \\
d5 & $\mathit{Candidate}
    \xrightarrow{\text{Personality test}} \mathit{Score_{PT}}$ 
    & External Provider & Personality testing & Interview \\
d6 & $[\mathit{RP}], \mathit{Candidate}, \mathit{Score_{PT}}
    \xrightarrow{\text{Interview}} \mathit{Score_{comp}}$ 
    & Recruiter, Hiring Manager & Competence-based interview & Interview \\
d7 & $\mathit{D_1} \times \mathit{Score_{comp}}
    \xrightarrow{\text{Ref. check}} \mathit{D_2}$ 
    & Recruiter & Reference and credential verification & Validation \\
d8 & $\mathit{D_2} 
    \xrightarrow{\text{Final selection}} \mathit{D_3}$ 
    & Hiring Manager & Final candidate selection & Selection \\
\bottomrule
\end{tabular}
}
\caption{Deployer recruitment process: channels, operations, and stages.}
\label{tab:deployer}
\end{table*}

\begin{table*}[t]
\centering
\scriptsize
\caption{Illustrative roles and legitimate questions in the Deployer recruitment
process. Roles reflect theoretical information demands and accountability obligations rather than job titles. Questions are illustrative and grounded in the case-study context.
Classes distinguish \textit{process structure}, \textit{specific run},
\textit{counterfactual}, \textit{statistics over runs}, and
\textit{derived oversight} questions. Parenthetical annotations indicate when
a question is evaluated specifically against preserved logs.}
\label{tab:roles}

\begin{tabularx}{\textwidth}{@{}p{0.85cm} X p{2.8cm}@{}}
\toprule
ID & \textbf{Example legitimate question} & \textbf{Class} \\
\midrule

\multicolumn{3}{l}{\textbf{Candidate} \quad
\small{\textit{Subject of assessment; seeks fair evaluation}}} \\
C1 & Was I hired, and at what step was I eliminated?
   & specific run \\
C2 & What criteria was I assessed against?
   & process structure \\
C3 & How did those criteria affect my outcome?
   & specific run \\
C4 & What would have changed my result? (structure/logs)
   & counterfactual \\
C5 & What happened in my assessment? (logs)
   & specific run \\

\midrule

\multicolumn{3}{l}{\textbf{Recruiter} \quad
\small{\textit{Manages process; assesses candidates}}} \\
R1 & What criteria should be assessed against?
   & process structure \\
R2 & Does this candidate meet the criteria?
   & specific run \\
R3 & What supports my assessment?
   & specific run \\
R4 & What did I base my scoring of this candidate on? (logs)
   & specific run \\
R5 & Can I verify that candidates were treated consistently and without discrimination?
   & derived oversight \\

\midrule

\multicolumn{3}{l}{\textbf{Hiring Manager} \quad
\small{\textit{Makes final selection; bears formal accountability}}} \\
HM1 & Does this candidate meet the job requirements?
    & specific run \\
HM2 & Which candidate should I select?
    & specific run \\
HM3 & Can I justify this candidate over the others?
    & specific run \\
HM4 & What was the basis of the final selection decision? (logs)
    & specific run \\
HM5 & Can I take responsibility for this decision, given how it was reached?
    & derived oversight \\

\midrule

\multicolumn{3}{l}{\textbf{HR Manager} \quad
\small{\textit{Equal-treatment compliance; organizational oversight}}} \\
HR1 & Is there an authoritative document setting out criteria/procedure?
    & process structure \\
HR2 & How was this decision reached? (logs)
    & specific run \\
HR3 & Were the criteria and procedures applied consistently? (logs)
    & statistics, runs \\
HR4 & Which channels touch personal data, and is that flow traceable?
    & process structure \\
HR5 & Is the process designed to treat everyone the same way?
    & process structure \\

\midrule

\multicolumn{3}{l}{\textbf{Appeals Assessor} \quad
\small{\textit{Evaluates contested decisions; reconstructs decision basis}}} \\
AA1 & What was the decision based on?
    & specific run \\
AA2 & Would the decision have differed under the facts raised? (structure/logs)
    & counterfactual \\
AA3 & Does the appellant's account match what happened? (logs)
    & specific run \\
AA4 & Was the designed procedure actually followed? (logs)
    & specific run \\

\bottomrule
\end{tabularx}
\end{table*}

\begin{figure*}[t]
\centering
\resizebox{\linewidth}{!}{%
\begin{tikzpicture}[x=1cm,y=1cm]

\node[mainnode,font=\small\bfseries] (open) at (1.3,2.85) {Open position};
\node[mainnode] (rp) at (2.8,1.85) {Requirements\\profile $[RP]$};

\draw[flow] (open) -- (rp);

\node[mainnode] (pot) at (0.0,0.0) {Potential\\Candidates};
\node[mainnode] (d0) at (2.9,0.0) {Candidates\\$D_0$};

\node[mainnode] (sr) at (5.55,0.90) {$Score_R$};
\node[mainnode] (shm) at (5.55,-0.90) {$Score_{HM}$};

\node[mainnode] (d1) at (8.55,0.0) {Shortlisted\\Candidates\\$D_1$};
\node[mainnode] (sc) at (11.2,0.0) {$Score$};
\node[mainnode] (d2) at (13.8,0.0) {Validated\\Candidates\\$D_2$};
\node[mainnode] (d3) at (16.35,0.0) {Final\\Candidates\\$D_3$};

\draw[flow] (pot.east) -- (d0.west)
  node[midway,above=3pt,op]{Application / CV}
  node[midway,below=4pt,chan]{d1};

\draw[flow] (rp.south) -- ++(0,-0.55) -| ($(d0.north)+(0,0.02)$);

\coordinate (split) at (4.15,0.0);

\draw[flow] (d0.east) -- (split);

\draw[flow] (split) |- (sr.west)
  node[pos=.72,above=4pt,op]{Initial scoring}
  node[pos=.72,below=5pt,chan]{d2};

\draw[flow] (split) |- (shm.west)
  node[pos=.72,below=4pt,op]{Initial scoring}
  node[pos=.72,below=16pt,chan]{d3};

\coordinate (merge) at (6.55,0.0);

\draw[flow] (sr.east) -| (merge);
\draw[flow] (shm.east) -| (merge);

\draw[flow] (merge) -- (d1.west)
  node[pos=.42,above=5pt,op]{Selection meeting}
  node[pos=.42,below=5pt,chan]{d4};

\draw[flow] (d1.east) -- (sc.west)
  node[midway,above=5pt,op]{Interview}
  node[midway,below=5pt,chan]{d6};

\draw[flow] (sc.east) -- (d2.west)
  node[midway,above=5pt,op]{Ref. check}
  node[midway,below=5pt,chan]{d7};

\draw[flow] (d2.east) -- (d3.west)
  node[midway,above=5pt,op]{Final}
  node[midway,below=5pt,chan]{d8};

\end{tikzpicture}%
}
\caption{An abbreviated information-flow model of the Deployer recruitment process prior to AI tool integration. For readability, the figure omits some secondary details represented in Table~\ref{tab:deployer}.}
\label{fig:deployer_ifm}
\end{figure*}
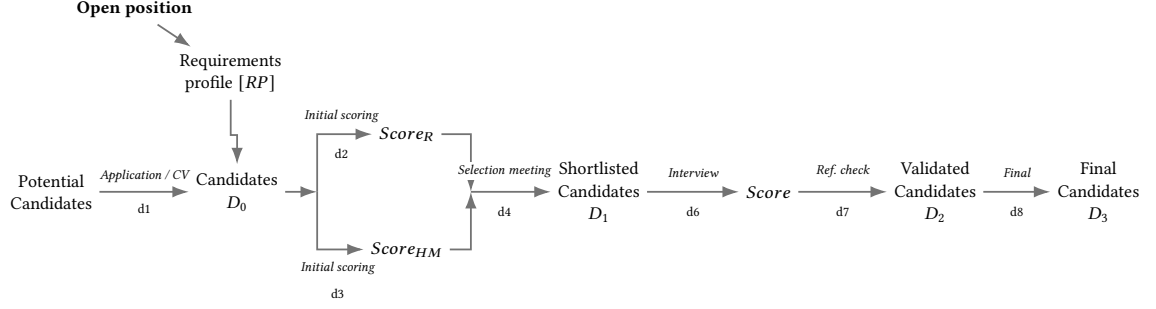

The Deployer's recruitment process follows organizational guidelines emphasizing competence-based assessment and equal treatment. Several variants of the recruitment process was used internally for either specialized positions and high-volume temporary positions. The process here presented as illustration is the process for senior position recruitment, which involves multiple assessment stages and several decision-makers. The process is governed by a requirement profile --- a document specifying the competencies and criteria relevant to the position --- which is constructed before recruitment starts and serves as an static 
input to assessment and interview stages.

The recruitment process operates as follows, as shown as an information network in Figure \ref{fig:deployer_ifm} and in Table \ref{tab:deployer}.

Candidates apply by submitting a CV through a candidate tracking software (d1). Recruiters and Hiring Manager independently score candidates against the requirement profile (d2--d3). Scored candidates are selected for a shortlist in a collaborative meeting between Recruiters and Hiring Manager (d4). The shortlisted candidates undergo personality testing via an external provider (d5) and a competence-based interview conducted by the Recruiter (d6). Reference and credential checks follow (d7) before the Hiring Manager makes the final selection based on Recruiters recommendation (d8).

The process was found to be notably informal at several stages. Initial scoring (d2--d3) used a 4-grade scale without structured documentation requirements. The selection meeting (d4) synthesize these scores together with the assessors' broader impressions of candidates, and the grounds for selection are not formally recorded. This informality has consequences for the traceability of decisions and is a complicating factor when it comes to AI tool integration, points we return to in Section~\ref{sec:role-dep} and the discussion in Section~\ref{sec:discussion}.

\subsubsection{Roles and obligations pre-deployment}

The stakeholders in this process operate under a combination of legal requirements, organizational policies, and professional responsibilities including public employment law, discrimination law, and administrative procedure requirements. These obligations predate and are independent of AI involvement and mainly form the information needs we term legitimate questions. In this case study we do not attempt to disentangle the full legal landscape but rather take the professional responsibilities of the interviewed stakeholders as a proxy for the obligations those roles carry. The following is illustrative rather than exhaustive: it draws on the case study material to demonstrate how obligations generate information needs and thereby illustrate the paper's theoretical contribution, not to enumerate all such needs comprehensively (for this case or otherwise).

Several distinct stakeholder roles can be identified, each carrying different responsibilities. The \textit{Candidate} goes through the process with an obligation to represent themselves accurately, and rights to be assessed fairly, to know how they are evaluated, and to be able to contest a decision they believe was unfair. The \textit{Recruiter} carries practical responsibility for managing the process and assessing candidates competently and equally. The \textit{Hiring Manager} takes formal responsibility for the hiring decision and must ensure the process follows organizational policy and that decisions can be justified and audited. The \textit{HR Manager} oversees equal treatment compliance across recruitment processes and must ensure the process is reconstructible if complaints arise. The \textit{Appeals Assessor} --- any person or function that evaluates a challenge to a recruitment decision --- must be able to reconstruct the basis of the original decision and determine whether it was correctly reached.

The information needs that arise from these roles and responsibilities can be organized by what kind of access to the process they require which can be matched to the formalism from Section \ref{sec:3_transparency_answerability}. 
Some needs concern the structure of the process itself: how it is designed, what criteria apply, what steps are involved, and what decision logic governs it. 
Other needs concern what actually happened in a specific case: what information was used, what scores were given, what decision was reached, and whether relevant attributes such as demographic characteristics are traceable through the run. 
Yet other require records of past cases: logs or documentation that allow a specific decision to be reconstructed or a pattern of decisions to be examined. 
Finally, some needs require aggregated understanding of how the process behaves across many cases.

These four kinds of access -- the structure of the process, what happened in a specific case, records of past cases, and patterns across many cases -- correspond to information the process must preserve and make available for the obligations and rights described above to be fulfillable. And even if they are available, the structure of the decision-making system need to support the questions asked. The formal treatment of these information levels, and the conditions under which they are accessible to a given role, is developed in Sections~\ref{sec:structural}--\ref{sec:role-dep}.
These responsibilities and rights generate information needs of different kinds,
illustrated in Table~\ref{tab:roles}. Not all such information needs can be
satisfied directly by the role that holds the corresponding right or responsibility. A candidate's right to understand the basis of a decision, for example, may depend on information held by the Hiring Manager or Appeals Assessor, and may correspond to an obligation on those roles to provide or ensure that account. Role-dependent answerability therefore depends not only on which questions are legitimate for a stakeholder, but also on which parts of the process a role can access and for which parts it carries relevant
obligations. This distinction is developed formally in
Section~\ref{sec:role-dep}. Also specific forms of access to information can be seen as a kind of role, e.g answering a question through persistent logs.

\subsection{The Provider System}

\begin{figure*}[t]
\centering
\resizebox{\linewidth}{!}{%
\begin{tikzpicture}[x=1cm,y=1cm]

\node[mainnode,font=\small\bfseries] (open) at (0.0,2.0) {Open position};
\node[actor] at (0.6,1.50) {Recruiter};
\node[mainnode] (cfg) at (0.0,0.7) {[Configuration]};

\draw[flow] (open.south) -- (cfg.north);

\node[mainnode] (c0) at (0.0,0.0) {Candidates\\$C_0$};
\node[mainnode] (c1) at (2.9,0.0) {Filtered\\$C_1$};
\node[mainnode] (ans) at (6.0,0.0) {Answers};
\node[mainnode] (score) at (9.8,0.0) {Scored\\Template};
\node[mainnode] (rank) at (13.3,0.0) {Ranked\\Candidates\\$C_2$};

\draw[flow] (c0.east) -- (c1.west)
  node[midway,above=3pt,op]{Hard req.\ filter}
  node[midway,below=4pt,chan]{p1};

\coordinate (qsplit) at (4.4,0.0);

\draw[flow] (c1.east) -- (qsplit);

\draw[flow] (qsplit) |- ($(ans.west)+(0,0.85)$)
  node[pos=.72,above=2pt,op]{Simple questions}
  node[pos=.72,below=3pt,chan]{p2};

\draw[flow] (qsplit) -- (ans.west)
  node[midway,above=2pt,op]{Free text questions}
  node[midway,below=4pt,chan]{p3};

\draw[flow] (qsplit) |- ($(ans.west)+(0,-0.85)$)
  node[pos=.72,below=2pt,op]{Competence questions}
  node[pos=.72,below=13pt,chan]{p4};

\coordinate (ssplit) at (7,0.0);

\draw[flow] (ans.east) -- (ssplit);

\draw[flow] (ssplit) |- ($(score.west)+(0,1.25)$)
  node[pos=.74,above=2pt,op]{AI rules}
  node[pos=.74,below=3pt,chan]{p5};

\draw[flow] (ssplit) |- ($(score.west)+(0,0.42)$)
  node[pos=.74,above=2pt,op]{Free text analysis}
  node[pos=.74,below=3pt,chan]{p6};

\draw[flow] (ssplit) |- ($(score.west)+(0,-0.42)$)
  node[pos=.74,above=2pt,op]{Competence analysis}
  node[pos=.74,below=3pt,chan]{p7};

\draw[fallback] (ssplit) |- ($(score.west)+(0,-1.25)$)
  node[pos=.74,above=2pt,op]{Human fallback}
  node[pos=.74,below=3pt,chan]{p8};

\draw[flow] (score.east) -- ++(1.55,0)
  node[midway,above=3pt,op]{Aggregation}
  node[midway,below=4pt,chan]{p9};

\coordinate (agg) at ($(score.east)+(1.55,0)$);

\draw[flow] (agg) -- (rank.west)
  node[midway,above=3pt,op]{Ranking}
  node[midway,below=4pt,chan]{p10};

\draw[
  decorate,
  decoration={brace,mirror,amplitude=4pt},
  draw=green!45!black
] (2.6,-1.55) -- (6.7,-1.55)
  node[midway,below=8pt,grouplabel]{AI Chatbot};

\draw[
  decorate,
  decoration={brace,mirror,amplitude=4pt},
  draw=green!45!black
] (7.2,-1.55) -- (10.8,-1.55)
  node[midway,below=8pt,grouplabel]{AI Scoring};

\end{tikzpicture}%
}
\caption{An abbreviated version of the information flow model of the Provider's AI system for recruitment. Inputs are a Client recruitment request, captured as a requirement configuration, and a set of candidates. The output is a list of ranked and scored candidates.}
\label{fig:provider_ifm}
\end{figure*}
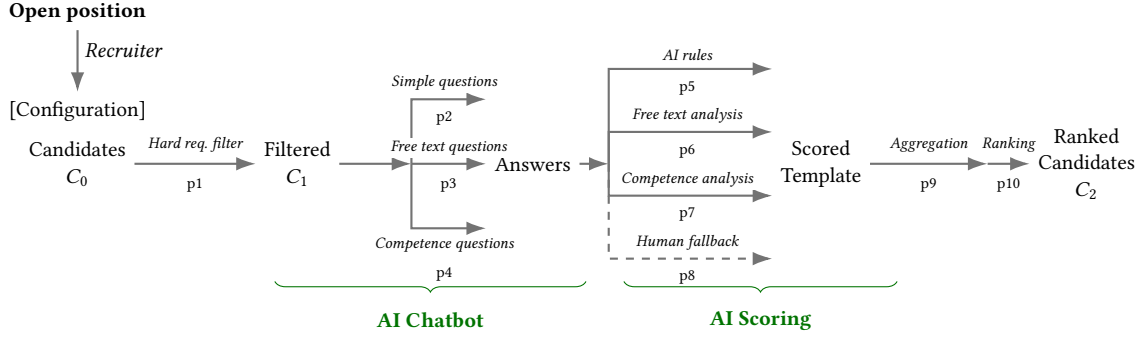

\begin{table*}[t]
\centering
\resizebox{0.95\textwidth}{!}{%
\begin{tabular}{cllll}
\toprule
\textbf{\#} & \textbf{Channel} & \textbf{Performed by} & \textbf{Operation} & \textbf{Stage} \\
\midrule
p1 & $[\mathit{Template}], \mathit{C_0} 
    \xrightarrow{\text{Hard req. filter}} \mathit{C_1}$ 
    & AI-Filter & Hard requirement filtering & Filtering \\
p2 & $[\mathit{TQ}], \mathit{Candidate} 
    \xrightarrow{\text{Simple questions}} \mathit{A_{rule}}$ 
    & Candidate & Rule-based questioning & Interview \\
p3 & $[\mathit{TQ}], \mathit{Candidate} 
    \xrightarrow{\text{Free text questions}} \mathit{A_{text}}$ 
    & Candidate & Open response collection & Interview \\
p4 & $[\mathit{TQ}], \mathit{Candidate} 
    \xrightarrow{\text{Competence questions}} \mathit{A_{comp}}$ 
    & Candidate & Competence-based questioning & Interview \\
p5 & $[\mathit{TQ}], \mathit{A_{rule}} 
    \xrightarrow{\text{AI rules}} \mathit{S_{rule}}$ 
    & AI-Scoring & Rule-based scoring & Scoring \\
p6 & $[\mathit{TQ}], \mathit{A_{text}} 
    \xrightarrow{\text{AI Free text}} \mathit{S_{text}}$ 
    & AI-NLP & Semantic text analysis & Scoring \\
p7 & $[\mathit{TQ}], \mathit{A_{comp}} 
    \xrightarrow{\text{AI Competence}} \mathit{S_{comp}}$ 
    & AI-NLP & Competence assessment & Scoring \\
p8 & $[\mathit{TQ}], \mathit{A_{comp}} 
    \xrightarrow{\text{Human fallback}} \mathit{S_{human}}$ 
    & Human Evaluator & Manual competence assessment & Scoring \\
p9 & $?\{\mathit{S_{rule}}, \mathit{S_{text}}, \mathit{S_{comp}} 
    \alt \mathit{S_{human}}\} 
    \xrightarrow{\text{Aggregation}} \mathit{ScoredTemplate}$ 
    & AI-Scoring & Score aggregation & Scoring \\
p10 & $\mathit{C_1} \times \mathit{ScoredTemplate} 
    \xrightarrow{\text{Ranking}} \mathit{C_2} \times \mathit{ScoredTemplate}$ 
    & AI-UI & Ranked candidate presentation & Presentation \\
\bottomrule
\end{tabular}
}
\caption{Provider AI system: channels, operations, and stages.}
\label{tab:provider}
\end{table*}

In this sub-section we describe the Provider's AI system, conceptually and as an information flow model in Table~\ref{tab:provider} and Figure~\ref{fig:provider_ifm}, following the same approach as for the Deployer system in the previous sub-section. 

The Provider system conducts AI-supported chat interviews with candidates and 
produces a ranked, scored list as output. The system is governed by an interview 
template --- a configurable structure jointly constructed by the Provider and 
Deployer --- which determines which questions are asked, which scoring models 
are activated, and how scores are weighted and aggregated. Each question in the 
template instantiates an independent scoring path through the system, drawing 
on a combination of rule-based algorithms and fine-tuned language models. The 
template is thus not merely a configuration parameter but the primary determinant 
of system behaviour at every stage.

The system operates at two levels of granularity. Filtering (p1) and ranking 
(p10) operate on the candidate pool. Interview and scoring (p2--p9) operate 
per candidate, producing a Scored Template for each candidate that captures 
their performance across all template questions independently. The final output 
is a ranked candidate list paired with Scored Templates, presented to the 
Deployer via an applicant tracking system.

One structural feature of the scoring stage warrants explicit note. The majority of candidates are assessed via automated scoring (p5--p7). A minority 
are routed to manual assessment by human evaluators (p8). These are mutually 
exclusive paths: a candidate assessed manually does not also receive automated 
scores. While the routing criterion is described in the system documentation, the output presented to the Deployer carries no indication of which path was taken for a given candidate. The recruiter and candidate both see the final total score but the recruiter in addition receive the full Scored Template, which contain scores per question.

\subsection{The Template as Shared Configuration}

Both systems are configured by a shared object --- the interview template --- 
produced in a preparatory flow that runs orthogonally to the candidate flow:

\begin{equation}
\mathit{Vacancy} \xrightarrow{\text{Req. profile}} \mathit{RP} 
\xrightarrow{\text{Template construction}} \mathit{Template}:[TQ_1, TQ_2, ..., TQ_n]
\end{equation}

The requirement profile $[\mathit{RP}]$ captures the Deployer's competency 
requirements for the position. The template $[\mathit{Template}]$ operationalises 
these requirements as interview questions and scoring criteria within the 
Provider's system. This shared object is the formal basis for the compositional 
relationship between the two systems, which we develop in 
Section~\ref{sec:compositional}.

\subsection{Analytical Scope}

The following three sections apply the framework to this case along the three dimensions of answerability. Section~\ref{sec:structural} examines what is answerable from the structure of each system in isolation. 

Section~\ref{sec:compositional} examines what further questions become answerable --- and what gaps become visible --- once the two systems are composed. 

Section~\ref{sec:role-dep} examines how answerability varies with stakeholder access, and states the full answerability (Proposition~\ref{prop:role-answerability}) and transparency relation (Definition~\ref{def:role-adequate-transparency}).

\section{Structural Answerability}
\label{sec:structural}
The Provider system $\mathbf{P}$ and Deployer process $\mathbf{D}$ introduced in Section~\ref{sec:use-case} are information networks in the sense of Definition~\ref{def:network}. This section examines what transparency questions 
are answerable from these networks directly. We call these \emph{structural questions}: questions whose answers are determined by the structure of $\mathbf{P}$ or $\mathbf{D}$ alone.

\subsection{What and How}

Since information networks model decision processes, their structure takes on the additional meaning of \emph{what} a system does and \emph{how} it does it. The most abstract account of what $\mathbf{P}$ does is given by $\mathit{In}(\mathbf{P}) \to \mathit{Out}(\mathbf{P})$:
\[
[\mathit{Template}], \mathit{C_0} \xrightarrow{\mathit{AI\ Assessment}} \mathit{C_2} \times \mathit{ScoredTemplate}
\]
\textbf{What?} $\mathbf{P}$ is an \emph{AI Assessment} from \emph{Candidates} to \emph{Scored Candidates} using an \emph{Interview Template}.

This answers the \emph{what} but says nothing about the \emph{how}. Expanding the channel \emph{AI Assessment} --- refining toward $\mathbf{P}^d$ --- answers progressively more specific how questions:
\[
[\mathit{Template}], \mathit{C_0} \xrightarrow{\text{Hard req. filter}} \mathit{C_1} \xrightarrow{\mathit{AI\ Interview}} \mathit{C_2} \times \mathit{Answers} \xrightarrow{\mathit{AI\ Evaluation}} \mathit{C_2} \times \mathit{ScoredTemplate}
\]
\textbf{How?} \emph{Candidates} are \emph{Filtered}, then \emph{Interviewed}, then \emph{Answers} are \emph{Scored}.

This \emph{what-by-how} pattern is general: any channel in an information network can be expanded into sub-channels that answer progressively more specific how questions, down to the level of granularity that the network supports. Each channel in Table~\ref{tab:provider} is an answer to a how question at its own level of granularity.

\subsection{Structural Questions over $\mathbf{P}$}
\label{subsec:structural-answerability}

The What/How reading of $\mathbf{P}$ established its overall structure. We now ask more specific questions of that structure. Each question below is expressed as an inquisitive network and answered by inspecting $\mathbf{P}$ directly.

$$\mathcal{Q}_{CV} : \textit{``Is the CV used in scoring?"} \implies [\mathit{CV} 
\to \textit{Scored T.}]? \implies \{L(\mathit{CV} 
\to \textit{Scored T.}\alt\neg L(\mathit{CV} 
\to \textit{Scored T.})\}$$

$\mathcal{Q}_{CV}$ asks if there is a channel from the CV to the score or not. The $join(\mathcal{Q}_{CV}) = \: !\mathcal{Q}_{CV} = \textit{Scored Template}$ and because the alternatives are disjoint the $meet(\mathcal{Q}_{CV}) = \bot$. We see that $\mathcal{P} \preceq \: !\mathcal{Q}_{CV}$ but \textbf{P} is only in one of the alternatives. Therefore the question is settled: A CV is not used for scoring. In fact the system $\mathbf{P}$ excludes CV's by design according to its developers.

\[
\mathcal{Q}_{auto} : \textit{``Is scoring automatic?"} \implies 
\{\mathit{Answers} \xrightarrow{x} \mathit{ScoredTemplate} \mid {?d(x)}\}
\]
\[
\mathcal{P} \cap \mathcal{Q}_{auto} =
\mathit{Answers} \xrightarrow{x} \mathit{ScoredTemplate} \mid d(x) = 
?\{\mathit{AI\ Rules} \alt \mathit{AI\ Free\ Text} \alt 
\mathit{AI\ Competence} \alt \mathit{Human\ Fallback}\}
\]

$\mathcal{Q}_{auto}$ asks for the description of the channel from 
\emph{Answers} to \emph{Scored Template}. Evaluating $\mathcal{Q}_{auto} 
\cap L(\mathbf{P})$: $\mathbf{P}$ contains four alternative descriptions 
of this channel, expressed at p9 as $?\{\mathit{S_{rule}}, \mathit{S_{text}}, 
\mathit{S_{comp}} \alt \mathit{S_{human}}\}$. The question is not settled 
to a single answer --- $\mathbf{P}$ spans multiple alternatives. The 
intersection reveals the full set of structural possibilities: scoring is 
one of four alternatives, three automatic and one manual. Which alternative 
applies to a specific candidate is not determined by $\mathbf{P}$ alone.

\[
p_6 = [\mathit{TQ}], \mathit{A_{text}} \xrightarrow{\mathit{AI\ Free\ Text}} \mathit{S_{text}}
\]
\[
\mathcal{Q}_{ft} : \textit{``How does free text scoring work?"} \implies L(p_6)/p_6 \cap \{\mathbf{P}|_{p_6}\}
\]

$\mathcal{Q}_{ft}$ asks whether $\mathbf{P}$ contains any internal structure for the free text scoring channel $p_6$ beyond its description as \emph{AI Free Text}. The projection $\mathbf{P}|_{p_6}$ extracts exactly the sub-network of $\mathbf{P}$ between $\mathit{A_{text}}$ and $\mathit{S_{text}}$. If $\mathbf{P}|_{p_6} = p_6$ --- no intermediate sites exist --- then $\mathbf{P}|_{p_6} \notin L(p_6)/p_6$ and the result is $\emptyset$. 
In the present case the free text scoring channel is not further defined in \textbf{P}. While this could be a lack of the information network \textbf{P}, it is worth noting that if $p_6$ is an opaque model not much more described internal structure might be available.

\subsection{Trace Questions}
\label{subsec:trace-questions}
Structural questions are answered from $\mathbf{P}$ alone. But there are many possible ways of realizing the structure $\mathbf{P}$. To ask questions about specific traces of the system we define an \textit{implementation} $\hat{\mathbf{P}}$.

\begin{definition}[Implementations and Traces]
\label{def:implementations}
An \emph{implementation} $\hat{\mathbf{N}}$ of a network $\mathbf{N}$ is a theoretical device with the following properties:
\begin{itemize}
    \item We assume that there exist an accessible actual accessible system $\omega \precsim \hat{\mathbf{N}}$ ($\omega$ is approximately described by $\hat{\mathbf{N}}$). 
    \item The representation $\hat{\mathbf{N}}$ defines all sites such that the description of each site $S \in \hat{\mathbf{N}}$ holds informative content $!d(S)$ which explains what information $S$ contains, and alternatives $?d(S)$ representing every possible value of $S$.
    \item $\hat{\mathbf{N}}$ is a device such that $\hat{\mathbf{N}} \preceq \mathbf{N}$, which given any settled input in $In(\hat{\mathbf{N}})$ produces settled values for all downstream sites, iterating through $\hat{\mathbf{N}}$ in topological order until all sites up to and including $Out(\hat{\mathbf{N}})$ are settled.
\end{itemize}

A \emph{trace} $\mathbf{t}$ of $\hat{\mathbf{N}}$ is a network where all sites have fully settled descriptions produced by $\hat{\mathbf{N}}$ for some settled input. A trace thus represent a record of a whole run through the decision process. The set of all traces of $\hat{\mathbf{N}}$ conditioned on an input set $I$ is:
\[
T(\hat{\mathbf{N}}, I) = \{ \mathbf{t} \mid \mathbf{t} \text{ is a trace of } \hat{\mathbf{N}} \text{ for some } \mathbf{t}_{in} \in I \}
\]
where $I$ is a set of settled inputs over $\mathit{In}(\mathbf{N})$. The unconditioned set $T(\hat{\mathbf{N}})$ is the special case where $I$ ranges over all possible inputs.
\end{definition}

Let $\hat{\mathbf{P}}$ be an implementation of $\mathbf{P}$ and $\hat{\mathbf{D}}$ be an implementation of $\mathbf{D}$. 

\[
\mathcal{Q}_{answers}^{\mathbf{t}} : \textit{``What were this candidate's answers?"} \implies \{s \in \mathbf{t} \mid s \in \{\mathit{A_{rule}}, \mathit{A_{text}}, \mathit{A_{comp}}\}\}
\]

\[
\mathcal{Q}_{answers} : \textit{``Candidate $c$ got score $s$, what were their answers?"} \implies c \xrightarrow{} ?\{\mathit{A_{rule}}, \mathit{A_{text}}, \mathit{A_{comp}}\} \xrightarrow{} s:\mathit{ScoredTemplate}
\]

$\mathcal{Q}_{answers}\cap \{\mathbf{t}\}$ asks for the settled values at all answer sites in trace $\mathbf{t}$. These sites are fully settled in any trace of $\hat{\mathbf{P}}$ --- every candidate who completes the interview produces settled values at each $\mathit{A_{rule}}, \mathit{A_{text}}, \mathit{A_{comp}}$. The question is answerable from $\mathbf{t}$ alone. Whether a trace of this kind is accessible to e.g the candidate or recruiter is dependent on:
\begin{itemize}
    \item If the structure of $P$, realized through $\hat{\mathbf{P}}$ supports the question - i.e is the information in the trace $\mathbf{t}$?
    \item Do the relevant stakeholders have access to the information in $\mathbf{t}$?
\end{itemize}
The second point is a matter of accessibility and is addressed in Section~\ref{sec:role-dep}, the first one is a matter of architecture and is either structurally allowed or not. We can note that from structure of $\textbf{P}$ all answers will be available in $\textbf{t}$ so there is no technical hindrance in this case. 

\[
\mathcal{Q}_{search}^{\mathbf t} :
\textit{``If I had answered differently, would I have scored higher?''}
\;\implies\;
?\{\, A_{comp} \mid \pi(A_{comp},\mathit{higher}) \in T(\hat{\mathbf P}) \,\}
\]

Assuming full access to $T(\hat{\mathbf P})$, counterfactual-style questions can be answered extensionally by lookup over settled traces. Since each $\mathbf t \in T(\hat{\mathbf P})$ is fully settled, this simply asks for the settled values of $A_{comp}$ occurring in traces whose downstream path reaches the refined outcome $\mathit{higher}$. 

Such full search is likely practically infeasible, but with full access to $\hat{\mathbf{P}}$ and its sub-channels a counterfactual can be answered by simply trying a different alternative on the partial downstream path.

\[
\mathcal{Q}_{cf}^{\mathbf t} :
\textit{``If I had answered $a_{cf}$ instead of $a$, what would I have scored?''}
\;\implies\;
?T(\hat{\mathbf P}|_{p7}, \{a_{cf}\})
\]

Questions of this kind are directly relevant to the right to meaningful explanation under Art.~86 of the AI Act, but the possibility of such counterfactual experimentation depends both on structure and operational access. That is, does that sub-structure exist and can I query or change an input at a particular point? This will be explored further in Section \ref{sec:role-dep}.

These structural and trace-level questions establish what \textbf{P} and \textbf{D} can each answer in isolation. The transparency questions that matter most --- how \textbf{P} affects decisions within \textbf{D}, what any given stakeholder can know of either --- require more than either system alone, and are the subject of the following section.

\section{Composition and Answerability}
\label{sec:compositional}

The Provider system $\mathbf{P}$ and Deployer process $\mathbf{D}$ have been analyzed independently. But the considerations that matter most to candidates, recruiters, and regulators are not questions about either system 
alone --- they are questions about what happens when $\mathbf{P}$ operates within $\mathbf{D}$. A candidate who asks whether the recruitment process treated them fairly cannot be answered from $\mathbf{P}$ or $\mathbf{D}$ 
alone. These kinds of questions requires a composed system.

\subsection{Composition}

In line with the use case, we will think of composition as the problem of situating an existing tool described as $\textbf{P}$ within a socio-technical environment $\textbf{D}$. For any practical purpose this is more involved than putting one graph into the other. Parts of the original decisionmaking structure must be changed or should be changed to accommodate the addition in a proper way.

Let $\textbf{D}_0$ be the subset of $\textbf{D}$ which is left unchanged. Since it would be possible to put changed parts back just as they were $\textbf{D} \preceq \textbf{D}_0$. From this kept subset the decision system of $\textbf{D}_0$ is refined into $\textbf{D}_1$ in such a way that it accommodates the addition of $\textbf{P}$. This can be done in many ways but $join(\textbf{D}_1, \textbf{D}) = \textbf{D}_0$ and there must be at least one $meet(\textbf{D}_1, \textbf{P})$.
From here one or more channels of $D_1$ is refined in such a way that it includes the entirety of $\textbf{P}$ producing the composition $\textbf{D}_P$ where $\textbf{P} \subset \textbf{D}_P$. Assuming $\textbf{D} \cap \textbf{P} = \varnothing$ then $join(\textbf{D}_P, D) = \textbf{D}_0$. 
We will assume that compositions in general have these compositional stages $\textbf{C}^D_P := \langle D_0, D_1, D_P \rangle$.

\begin{definition}[Composition]
\label{def:composition}
Let $\mathbf{P}$ and $\mathbf{D}$ be information networks with 
$\mathbf{S}_\mathbf{P} \cap \mathbf{S}_\mathbf{D} = \emptyset$. A 
\emph{composition} of $\mathbf{P}$ into $\mathbf{D}$ is a triple 
$\textbf{C}^\mathbf{D}_\mathbf{P} := \langle \mathbf{D}_0, \mathbf{D}_1, 
\mathbf{D}_P \rangle$ where:
\begin{itemize}
    \item $\mathbf{D}_0$ is the \emph{preserved} sub-network of $\mathbf{D}$ 
    --- the part left unchanged by the integration.
    \item $\mathbf{D}_1$ is the \emph{adapted} network, accommodating $\mathbf{P}$. 
    There must exist at least one $\text{meet}(\mathbf{D}_1, \mathbf{P})$. 
    \item $\mathbf{D}_P$ is the \emph{composed} network including $\textbf{P}$.
\end{itemize}
Here: 
\[
\mathbf{D} \preceq \mathbf{D}_0, \quad 
\mathbf{D}_1 \preceq \mathbf{D}_0, \quad 
\mathbf{D}_P \preceq \mathbf{D}_1
\]
and $\mathbf{P} \subset \mathbf{D}_P$, $\text{join}(\mathbf{D}_1, \mathbf{D}) 
= \mathbf{D}_0$, and $\text{join}(\mathbf{D}_P, \mathbf{D}) = \mathbf{D}_0$.
\end{definition}

\subsection{Integration Scenarios}

We consider three instantiations of $\textbf{C}^\mathbf{D}_\mathbf{P}$ for the case 
study, differing in how much of $\mathbf{D}$ is preserved and how deeply 
$\mathbf{P}$ is integrated into the Deployer's process. These are summarised 
in Table~\ref{tab:scenarios}.

\begin{table*}[t]
\centering
\resizebox{0.95\textwidth}{!}{%
\begin{tabular}{lllll}
\toprule
\textbf{Scenario} & \textbf{$\mathbf{D}_0$ removes} & \textbf{$\mathbf{D}_1$ adaptation} & \textbf{$\mathbf{P}$ replaces} & \textbf{Human evaluation} \\
\midrule
$\mathbf{Min}^\mathbf{D}_\mathbf{P}$  & d2, d3 & $\mathit{C_1}$ feeds d2, d3 & d1 & All candidates passing hard req. \\
\midrule
$\mathbf{Alt}^\mathbf{D}_\mathbf{P}$  & d2, d3, d4 & Selection meeting receives & d1, d2, d3 & All candidates, \\
  &            & $\mathit{C_2} \times \mathit{ScoredTemplate}$ & & AI-ranked input \\
  &            & instead of $\mathit{Score_R}, \mathit{Score_{HM}}$ & & \\
\midrule
$\mathbf{Max}^\mathbf{D}_\mathbf{P}$  & d2, d3, d4 & Threshold channel added, & d1, d2, d3 & $\mathit{top}_x$ only \\
  &            & producing $\mathit{top}_x$, $\mathit{rest}$; & & proceed to d5, d6 \\
  &            & d5, d6 receive $\mathit{top}_x$ only & & \\
\bottomrule
\end{tabular}
}
\caption{Three integration scenarios as instantiations of $C^\mathbf{D}_\mathbf{P}$.}
\label{tab:scenarios}
\end{table*}

\paragraph{$\mathbf{Min}^\mathbf{D}_\mathbf{P}$ --- Minimal integration.}
$\mathbf{P}$ replaces only the application and hard requirement filtering 
channel d1. The preserved sub-network $\mathbf{D}_0$ retains all human 
assessment channels --- initial scoring, selection meeting, personality 
testing, competence interview, and final selection. The adaptation 
$\mathbf{D}_1$ connects the filtered candidate pool $\mathit{C_1}$ from 
$\mathbf{P}$ into the downstream channels of $\mathbf{D}$. All candidates 
passing hard requirements proceed to full human assessment. The AI system 
acts as a structured intake mechanism, not an assessor.

\paragraph{$\mathbf{Alt}^\mathbf{D}_\mathbf{P}$ --- Assessment integration.}
$\mathbf{P}$ replaces the application, hard requirement filtering, and the 
initial independent scoring channels d1--d3. The selection meeting d4 is 
retained but adapted: it now receives $\mathit{C_2} \times 
\mathit{ScoredTemplate}$ from $\mathbf{P}$ in place of the recruiter and 
hiring manager's independent 4-grade scores. The adaptation $\mathbf{D}_1$ 
restructures the selection meeting's inputs accordingly. All candidates 
proceed to the selection meeting, but the information available there has 
changed structurally --- the independent human pre-assessment is absent, 
replaced by the AI-generated ranking and scored template.

\paragraph{$\mathbf{Max}^\mathbf{D}_\mathbf{P}$ --- Full integration.}
$\mathbf{P}$ replaces d1--d3 as in $\mathit{Alt}^\mathbf{D}_\mathbf{P}$, 
but the selection meeting is also removed from $\mathbf{D}_0$. The 
adaptation $\mathbf{D}_1$ introduces a threshold channel that splits 
$\mathit{C_2} \times \mathit{ScoredTemplate}$ into two sites: 
$\mathit{top}_x$, who proceed to personality testing d5 and competence 
interview d6, and $\mathit{rest}$, who are directly excluded. 

\subsection{Decision Attribution} \label{sec:composition_attribution}

A natural way of evaluating compositions between AI and human decision systems $C^\mathbf{D}_\mathbf{P}$ 
is: where are decisions taken and who actually makes them? The answer is not simply a matter of policy or intent --- it can be analyzed as a structural property of $\mathbf{D}_P$, readable from the upstream paths of outcome sites.

Like $\textbf{D}$, $\mathbf{D}_P$ is a selection process starting with the full candidate pool 
$\mathit{C_0}$ and progressively excluding candidates at each filtering 
channel, with only the hired candidates remaining at $\mathit{D_3}$. Each 
filtering channel partitions its input into two cohorts: those who proceed 
($\mathit{top}$) and those who are excluded ($\mathit{rest}$). This 
partitioning maps $\mathbf{D}_P$ naturally onto the structure of a Sankey diagram (where the width of each flow represents the proportion of candidates at each stage).

The primary filtering points across the three scenarios are:
\begin{itemize}
    \item \textbf{Hard requirements} --- binary exclusion on mandatory criteria
    \item \textbf{Ranking} --- ordering and threshold on AI assessment scores
    \item \textbf{Selection meeting} --- collaborative human shortlisting
    \item \textbf{Post-interview} --- competence and personality assessment
    \item \textbf{Vetting} --- reference and credential checks
    \item \textbf{Final hiring decision} --- hiring manager selection
\end{itemize}

On each of these a smaller subset of candidates actually participates, so the \textit{rest} fraction is important for understanding the properties of the decision process.

These points of decision are not unrelated and participate differently in the different scenarios. Depending on their positioning the information on which they act are different.

\paragraph{$\mathbf{Min}^\mathbf{D}_\mathbf{P}$.}
Here the full filter chain from $\mathbf{D}$ remains in $\uparrow D_3$. To contrast with the other scenarios, we will review what the main filtering steps are basing their decisions on:
\begin{itemize}
    \item \textbf{Selection meeting:} Here hiring manager and recruiters have independently scored CVs on a rough 1-4 scale, but more importantly they were the ones doing the scoring and carry their opinions of the candidates with them into the meeting. The 1-4 scale is rough by design as it is not the main information carrier. 

    \item \textbf{Post-interview filtering:} The post-interview filtering is based on competence scores, the personality test scores and by the human judgment on the candidates.

    \item \textbf{Final hiring decision:} The hiring manager selects from a small shortlist of candidates who have passed the full process. The decision is a human judgment but operates the small subset of remaining candidates. The hiring decisions are made on recommendation from the hiring process.
\end{itemize}

\paragraph{$\mathbf{Alt}^\mathbf{D}_\mathbf{P}$.}
In $\mathit{Alt}^\mathbf{D}_\mathbf{P}$ the selection meeting remains but its information base changes fundamentally. The independent pre-scoring by recruiters and hiring manager is replaced by $\mathit{C_2} \times 
\mathit{ScoredTemplate}$ from $\mathbf{P}$. The key differences from $\mathit{Min}^\mathbf{D}_\mathbf{P}$ are:

\begin{itemize}
    \item \textbf{Selection meeting:} The meeting now operates on an     AI-generated ranking rather than independent human assessments. Participants have not scored the candidates themselves and do not carry 
    independent opinions into the meeting. The natural anchor for discussion is the AI ranking, which structurally favours automation bias --- accepting the ranking requires no effort, overriding it requires     
    consulting the full $\mathit{ScoredTemplate}$ for each candidate. 
    The $\mathit{rest}$ cohort at this stage is therefore more likely to 
    reflect AI judgment than in $\mathit{Min}^\mathbf{D}_\mathbf{P}$.

    \item \textbf{Post-interview and final decision:} These remain 
    structurally identical to $\mathit{Min}^\mathbf{D}_\mathbf{P}$, but 
    operate on a population pre-shaped by the AI ranking. The human 
    judgment at these stages is real but conditioned on who the AI 
    selected for consideration.
\end{itemize}

\paragraph{$\mathbf{Max}^\mathbf{D}_\mathbf{P}$.}
In $\mathit{Max}^\mathbf{D}_\mathbf{P}$ the selection meeting is removed 
entirely. The ranking channel produces $\mathit{top}_x$ and $\mathit{rest}$ 
directly. The key structural consequences are:

\begin{itemize}
    \item \textbf{Ranking as primary filter:} The $\mathit{rest}$ cohort 
    at the ranking stage has $\uparrow_{\mathbf{D}_P}(\mathit{rest})$ 
    containing only $\mathbf{P}$ channels --- no human channel reaches 
    them. These candidates are excluded on AI judgment alone, unseen by 
    any human assessor.

    \item \textbf{Post-interview and final decision:} These operate only 
    on $\mathit{top}_x$ --- a population entirely determined by 
    $\mathbf{P}$. Human judgment is genuine within this population but 
    cannot reach or reconsider candidates in $\mathit{rest}$. The 
    effective scope of human decision-making is structurally bounded by 
    the AI ranking threshold.
\end{itemize}

The structural differences across the three scenarios are summarized by 
comparing $\uparrow_{\mathbf{D}_P}(\mathit{rest})$ at the ranking stage: 
in $\mathit{Min}^\mathbf{D}_\mathbf{P}$ there is no AI-driven $\mathit{rest}$ 
cohort; in $\mathit{Alt}^\mathbf{D}_\mathbf{P}$ the $\mathit{rest}$ cohort 
at the selection meeting is AI-influenced but human-mediated; in 
$\mathit{Max}^\mathbf{D}_\mathbf{P}$ the $\mathit{rest}$ cohort up to the 
interview stage is entirely AI-determined. 
This decision gradient is a structural fact about the shape of $\mathbf{D}_P$.

\subsection{Effective Input, Impact and the Framing Trap}

A recurring failure in AI assessment is presuming that properties established within a component's scope hold beyond it~\cite{selbst2019fairness}. A fairness statistic computed over the outputs of a scoring model says little about whether that model's inputs carry protected-attribute proxies, nor whether its outputs produce discriminatory effects downstream. The same pattern arises in any domain where a subsystem is evaluated in isolation from its deployment context. To avoid this trap the AI-model frame needs to be situated in its deployment context, something that require at least answering (1) where the model is within the larger scope, (2) how the models output affect downstream consequences, (3) how the models inputs are obtained and (4) where these inputs originate from, and relate to society.

These conditions are essentially transparency questions. Let $c$ denote the channel in $\mathbf{D}_P$ corresponding to the AI model component of interest. Then:
\begin{enumerate}
\item \emph{Location}: $c \in \mathbf{C}_{D_P}$ — the component is a
    locatable channel in the composed network.
\item \emph{Output path}: $\pi(c, \mathit{Out}(\mathbf{D}_P)) \neq \emptyset$
    — there is a traceable path from the component's outputs to consequential
    outcome sites, such as the $[top_x, rest]$ partitions from subsection \ref{sec:composition_attribution}.
\item \emph{Input path}: $\pi(\mathit{In}(\mathbf{D}_P), c) \neq \emptyset$
    — there is a traceable path from the composed network's inputs to the
    component's inputs, capturing how the information the component receives
    comes to take the values it does.
\item \emph{Input population and static datasets}: the input sites $\mathit{In}(c)$ carry descriptions
    $d(s)$ whose alternatives $\mathbf{?}d(s)$ include candidate demographic
    or population characteristics. More generally, the sites in $In(\textbf{D}_P)$ are per definition populated from outside the scope of $\textbf{D}_P$ and may thus hide sampling bias. 
\end{enumerate}

In $\mathbf{D}_P$, conditions (1)--(3) are answerable for the AI scoring channels $p5$--$p7$ across all three integration scenarios. These channels are situated within $\mathbf{P} \subset \mathbf{D}_P$, so condition~(1) is met
structurally. It is not a given that it would be met -- if $\textbf{P}$ had not been described to this level $\textbf{P}$ could have been mistaken for the `AI model', in which case the wrong conclusions would have been reached. It is possible that channel $p6$ for instance can be further refined to reveal more structure - but the deeper into the system the channel is the more constrained it is. 

The input path is the upstream feeding into the models. In all integration scenarios the candidates are sourced the same way through a third party ATS fed by public job adverts, and while this is information which partially answers (4) this information is not in $D_P$. It is a general fact that any model has some frame and information outside it must be obtained through external context. 

Once in the system the candidates feeds the interview channels $p2-p4$, where candidate answers are produced under Template-governed questioning, satisfying condition~(3) and revealing that
the inputs to the scoring channels are themselves shaped by prior AI-candidate interaction. If these interactions are language dependent or otherwise non-inclusive, then the answers themselves -- even if they contain just anonymous text -- become demographic proxies. In this real use case, this was partly handled by a language setting.
The input path is the same in all three integration scenarios.

The output path, in contrast, is different in each scenario and thus the potential for effects on stakeholder outcomes differs from scenario to scenario -- despite models and indeed $\textbf{P}$ is identical. This highlights the importance of framing and importance of making sure stakeholder outcomes are in the modeled scope.

\subsection{Decision Semantics}

The meaning of a decision is not carried by its value alone but by the path that produced it. Two candidates with identical values at $\mathit{ScoredTemplate}$ may have arrived there through structurally different paths --- one assessed by AI scoring, one by human evaluation. These are not the same outcome despite identical surface values. The semantic content of a settled site $b$ in $\mathbf{D}_P$ is determined by $\uparrow_{\mathbf{D}_P}(b)$ --- the full upstream path that produced it. This has significant effects on how AI output should be understood, as the path through which it was reached must be taken into account.

A recruiter operating in $\mathit{Alt}^\mathbf{D}_\mathbf{P}$ or $\mathit{Max}^\mathbf{D}_\mathbf{P}$ who receives $\mathit{ScoredTemplate}$ values may reasonably assume they are looking at AI-generated scores generated through the specified path in $\mathbf{P}$. For the majority of candidates this is correct. But for candidates whose answers were routed to human evaluation --- the $?\{\mathit{S_{auto}} ; \mathit{S_{human}}\}$ alternative in $\mathbf{P}$ --- the upstream of $\mathit{ScoredTemplate}$ is structurally different. For instance, in human scoring, answers to earlier questions matter for later questions; in AI scoring they are independent. The recruiter's mental model of what they are examining is thus incorrect for this subset of outputs, and the actual system provides no signal to correct it.

The cases above concern misinterpretation of a correctly modeled alternative path. A related but distinct problem arises when $\omega$ itself is not properly described by $\hat{\mathbf{P}}$ --- when the actual system follows a path that the model does not sanction at all.

More generally, path-dependent meaning extends to all outcome sites: how we interpret that a candidate was rejected, shortlisted, or hired depends on which paths produced those outcomes. A special case arises when $\omega$ is not properly described by $\hat{\mathbf{P}}$. Consider if there is a bias or other systematic error introducing a hidden alternative path not modeled by a channel $A \to B$ in $\hat{\mathbf{P}}$. If so, the apparent downstream site $B$ is misinterpreted as having upstream $\uparrow B$ when values were in fact derived through some other path.

Consider
$$[\mathit{RP}], \mathit{D_0} \times \mathit{Score_R}, \mathit{Score_{HM}} \xrightarrow{\text{Selection meeting}} \mathit{D_1}[\mathit{top}_x, \mathit{rest}]$$
from the use case. Suppose the hiring manager lets his friends through more easily. Then such a candidate is misinterpreted as having passed the \textit{Selection meeting} when they in fact passed through a different channel $D_0 \xrightarrow[]{\text{Friend}} D_1[\mathit{top}_x]$. 
This observation admits a substantial generalization. Any systematic deviation of $\omega$ from $\hat{\mathbf{P}}$ --- bias, error, or unmodeled behavior --- is in principle expressible within the same formalism as an alternative sub-network competing with the modeled one. We state this as a conjecture:

\begin{conjecture}[Flow Substitution]
Let $\mathbf{P}$ be a well-formed information network and $\hat{\mathbf{P}}$ an implementation with $\omega \precsim \hat{\mathbf{P}} \precsim \mathbf{P}$. 
Suppose that for some inputs, $\omega$ produces traces $\mathbf{t} \notin T(\hat{\mathbf{P}})$ --- that is, the actual system deviates from the model. 
Then there exists a sub-network $\mathbf{N} \subseteq \mathbf{P}$ and an alternative network $\mathbf{N}_e$ such that $\mathit{In}(\mathbf{N}_e) = \mathit{In}(\mathbf{N})$, $\mathit{Out}(\mathbf{N}_e) = \mathit{Out}(\mathbf{N})$, and $\mathit{In}(\mathbf{N}_e) \subseteq \uparrow_{\mathbf{P}}(\mathit{Out}(\mathbf{N}))$ --- an alternative path draws on sites upstream in $\mathbf{P}$ to produce $Out(N)$ on a different path. 

The resulting $\mathcal{P}' = \mathbf{P}[?\{\mathbf{N} \alt \mathbf{N}_e\}/\mathbf{N}]$ is not a refinement of $\{\mathbf{P}\}$ but is a closer approximation of $\omega$.
\end{conjecture}

A proof is largely immediate: if the behaviour of $\omega$ under the deviation is modellable as a network at all, then the alternative $?\{\mathbf{N} ; \mathbf{N}_e\}$ captures it by construction. What is non-trivial, and what makes the conjecture substantive, is what $\mathbf{N}_e$ must look like to produce the observed deviation --- which sites it must draw on, which channels it must substitute for, and what those structural requirements reveal about the mechanism of the deviation itself. The conjecture thus frames an analytical move rather than a missing theorem: any systematic deviation of a system from its model is expressible as a structural alternative, and constructing that alternative is itself a diagnostic exercise.

\section{Role-Dependent Answerability}
\label{sec:role-dep}

The answerability results of Sections~5 and~6 treat $\mathbf{D}_P$ as a single shared object. In practice, different roles engage with different parts of the system under different conditions of access. A Candidate, a Recruiter, a Provider, and a regulatory auditor each bring distinct questions to the system and each has access to a different fragment of it. This section formalizes that variation.

\begin{definition}[Role]
\label{def:role}
A \emph{role} is a triple $R = \langle Q_R, S_R, V_R \rangle$ where:
\begin{itemize}
    \item $Q_R$ is a set of transparency questions role $R$ legitimately needs answered;
    \item $S_R \subseteq \mathbf{D}_P$ is a partially described sub-network of $\mathbf{D}_P$, representing the structural knowledge of $R$ --- what $R$ knows of the system through documentation, training, or communicated system descriptions;
    \item $V_R \subseteq \hat{\mathbf{D}}_P$ is a partially described sub-network of an implementation $\hat{\mathbf{D}}_P$, representing the operational access of $R$ --- the sites whose settled values are available to $R$ through logs, interfaces, or traces of the running system.
\end{itemize}
\end{definition}

\begin{definition}[Structural and Operational Question Classes]
\label{def:answerability-classes}
Let $R = \langle Q_R, S_R, V_R \rangle$ be a role. We write
$Q_R^S \subseteq Q_R$ for the set of structural questions of $R$, i.e.\ questions about the structure of the system such as paths, alternatives, dependencies, and channel existence, and
$Q_R^V \subseteq Q_R$ for the set of operational questions of $R$, i.e.\ questions about settled values in actual traces or runs or statistics thereof.
\end{definition}

The distinction between $S_R$ and $V_R$ is fundamental. $S_R$ captures what $R$ can reason about structurally --- the sites and channels $R$ knows to exist and can interpret. $V_R$ captures what $R$ can actually read in a running system. The difference is that between a manual in a drawer and a wall of live gauges: the manual is necessary to interpret the gauges, but the gauges cannot be replaced by the manual. Structural questions --- about paths, alternatives, and what the system is --- are in principle answerable from $S_R$ alone. Questions about specific traces --- what value did this site take, was this candidate scored by AI or human --- require $V_R$.

\begin{definition}[Maximal Descriptive Network]
Let $\mathbf{N}$ be a partially described information network. The \emph{maximal descriptive network} of $\mathbf{N}$, written $mdn(\mathbf{N})$, is the unique $\preceq$-maximal fully described network derivable from $\mathbf{N}$.
\end{definition}

\begin{lemma}[Collapse under ambiguous mediation]
\label{lem:ambiguous-collapse}
Let $\mathbf{N}$ be a partially described network, and suppose a dependency from $A$ to $B$ is mediated through both a described site $X$ and at least one undescribed path. Then $X$ cannot be retained in $mdn(\mathbf{N})$ as an intermediate on the path from $A$ to $B$.
Hence the ambiguous fragment collapses in $mdn(\mathbf{N})$ to the weaker fully described dependency
\[
A \to B, A \to X.
\]
Here $X$, if it is kept at all, is reduced to a side effect.
\end{lemma}

\begin{proof}
Retaining $X$ as an intermediate would resolve the mixed $A \to B$ dependency into the stricter described refinement $A \to X \to B$. But this is not derivable: by \textsc{Site Split-D}, a descriptive decomposition is allowed only when all component descriptions are present. Since one alternative path is undescribed, the ambiguity cannot be resolved in a fully described refinement. $mdn(\mathbf{N})$ may not transfer the undescribed internal structure onto remaining described structure and thereby grant it a stronger role than warranted. Therefore the strongest fully described dependency that remains is $A \to B$ with $A \to X$ as an non-causal side-effect.
\end{proof}
\begin{example}
Suppose a role has access to the partially described fragment
\[
A \to X_1 \to B,\ \dots,\ A \to X_6 \to B,\ A \to S \to B,\ B \to C,
\]
where $X_1,\dots,X_6,B,C$ are described but $S$ is not. Then the contribution of each $X_i$ to $B$ is ambiguous, since the undescribed path through $S$ may mediate or override the production of $B$. The conservative collapse is therefore
\[
A \to X_1,\ \dots,\ A \to X_6,\ A \to B \to C,
\]
rather than a network in which the $X_i$ remain answerability-relevant intermediates for $B$.
\end{example}

This matters directly for role-relative access. A role may observe a scattered set of described sites within a larger partially hidden process and still be unable to answer questions about how downstream outcomes are produced. The reason is that once the accessible fragment is collapsed to $mdn(S_R)$ or $mdn(V_R)$, ambiguously mediated interior structure may disappear. Role-relative access is therefore conservative with regards to answerability: visible fragments do not by themselves license structural conclusions about the channels they appear to participate in.

Proposition~\ref{prop:role-answerability} and Definition~\ref{def:role-adequate-transparency} capture two complementary aspects of this analysis. Role-dependent answerability is a property of the system and access structure. For a given role $R$ and question $q$, it identifies whether the available access suffices to settle $q$. Role-adequate transparency is a property of the role's epistemic position. It evaluates whether the system places $R$ in an adequate position across the full information demand $Q_R$.

\begin{proposition}[Role-Dependent Answerability]
\label{prop:role-answerability}
Let $R = \langle Q_R, S_R, V_R \rangle$ be a role.
\begin{enumerate}
    \item A question $q \in Q_R^S$ is \emph{structurally answerable} for $R$ iff
    \[
    \mathit{mdn}(S_R) \subseteq \mathcal{L}(q).
    \]
    \item A question $q \in Q_R^V$ is \emph{operationally answerable} for $R$ iff
    \[
    \mathit{mdn}(V_R) \subseteq \mathcal{L}(q).
    \]
\end{enumerate}
Thus structural answerability is determined by the role's accessible described sub-network, while operational answerability is determined by the role's accessible described sub-implementation.
\end{proposition}

\begin{definition}[Role-Adequate Transparency] \label{def:role-adequate-transparency}
A system is \emph{transparent} to role $R$ relative to $Q_R$ iff the access made available to $R$ is adequate to answer the questions in $Q_R$:
\[
\forall q \in Q_R^S,\ \mathit{mdn}(S_R) \subseteq \mathcal{L}(q)
\quad\text{and}\quad
\forall q \in Q_R^V,\ \mathit{mdn}(V_R) \subseteq \mathcal{L}(q).
\]
\end{definition}

\subsection{Examples of Role-Dependent Perspectives}

\begin{table*}[t]
\centering
\scriptsize
\caption{Answerability of legitimate questions before AI integration and across
the three integration scenarios. A = answerable, P = partially answerable,
B = blocked; parentheses indicate deferral to another role, and X/Y
structural/operational answerability. $\dagger$ marks limitations arising from
the undisclosed human fallback path, including uncertainty about when it was
used and the consequences of outsourced human processing. $\ddagger$ marks
limitations arising from informal note taking of decisions and meetings.
$\text{Max}^D_P$ is split by candidate cohort of \textit{top} ranking and the
\textit{rest} where needed; -- means missing presuppositions where the question
is not applicable.}
\label{tab:answerability-across-scenarios}

\begin{tabularx}{\textwidth}{@{}l X c c c c c@{}}
\toprule
& & & & & \multicolumn{2}{c}{$\text{Max}^D_P$} \\
\cmidrule(l){6-7}
Role & Legitimate question & Pre & $\text{Min}^D_P$ &
$\text{Alt}^D_P$ & \textit{rest} & \textit{top} \\
\midrule

\multirow{5}{*}{Candidate}
& C1: Was I hired, and at what step was I eliminated?
  & A & A & A & \multicolumn{2}{c}{A} \\
& C2: What criteria was I assessed against?
  & P & P & P & \multicolumn{2}{c}{P} \\
& C3: How did those criteria affect my outcome?
  & (A) & (A) & (P) & \multicolumn{2}{c}{(P$^\dagger$)} \\
& C4: What would have changed my result? (structure/logs)
  & (A/P$^\ddagger$) & (A/P$^\ddagger$) & (A/B)
  & \multicolumn{2}{c}{(P$^\dagger$/B)} \\
& C5: What happened in my assessment? (logs)
  & (P$^\ddagger$) & (P$^\ddagger$) & (P)
  & \multicolumn{2}{c}{(P$^\dagger$)} \\

\midrule

\multirow{5}{*}{Recruiter}
& R1: What criteria should be assessed against?
  & A & A & A & \multicolumn{2}{c}{A} \\
& R2: Does this candidate meet the criteria?
  & A & A & P & B & P \\
& R3: What supports my assessment?
  & A & A & P & \multicolumn{2}{c}{P} \\
& R4: What did I base my scoring of this candidate on? (logs)
  & P$^\ddagger$ & P$^\ddagger$ & P
  & \multicolumn{2}{c}{P$^\dagger$} \\
& R5: Can I verify that candidates were treated consistently and without discrimination?
  & A/A & A/A & A/P & B/B & P/B \\

\midrule

\multirow{5}{*}{Hiring Mgr.}
& HM1: Does this candidate meet the job requirements?
  & A & A & A & -- & A \\
& HM2: Which candidate should I select?
  & A & A & P & -- & P \\
& HM3: Can I justify this candidate over the others?
  & A & A & P & -- & P \\
& HM4: What was the basis of the final selection decision? (logs)
  & P$^\ddagger$ & P$^\ddagger$ & P & -- & P$^\ddagger$ \\
& HM5: Can I take responsibility for this decision, given how it was reached?
  & A/A & A/A & A/P & B/B & P/B \\

\midrule

\multirow{5}{*}{HR Manager}
& HR1: Is there an authoritative document setting out criteria/procedure?
  & A & A & A & \multicolumn{2}{c}{A} \\
& HR2: How was this decision reached? (logs)
  & B$^\ddagger$ & B$^\ddagger$ & P
  & \multicolumn{2}{c}{A} \\
& HR3: Were the criteria and procedures applied consistently? (logs)
  & P$^\ddagger$ & P$^\ddagger$ & P$^\ddagger$
  & \multicolumn{2}{c}{P$^\dagger$} \\
& HR4: Which channels touch personal data, and is that flow traceable?
  & A & A & B$^\dagger$
  & \multicolumn{2}{c}{B$^\dagger$} \\
& HR5: Is the process designed to treat everyone the same way?
  & A & A & P$^\dagger$
  & \multicolumn{2}{c}{P$^\dagger$} \\

\midrule

\multirow{4}{*}{Appeals Assessor}
& AA1: What was the decision based on?
  & (A) & (A) & P & \multicolumn{2}{c}{P} \\
& AA2: Would the decision have differed under the facts raised? (structure/logs)
  & (A/P$^\ddagger$) & (A/P$^\ddagger$) & A/B
  & \multicolumn{2}{c}{A/B} \\
& AA3: Does the appellant's account match what happened? (logs)
  & (A) & (A) & P & \multicolumn{2}{c}{P} \\
& AA4: Was the designed procedure actually followed? (logs)
  & B$^\ddagger$ & B$^\ddagger$ & P
  & \multicolumn{2}{c}{P$^\dagger$} \\

\bottomrule
\end{tabularx}
\end{table*}

\paragraph{The Candidate.} A candidate completes the AI-supported interview and is rejected. Their $V_R$ is richest precisely where $\mathbf{D}_P$ is least visible to other roles: they know what they said and how the interaction felt. But their $S_R$ is sparse downstream of p4, and the natural refinement of their $S_R^d$ --- a sequential AI assessor scoring responses in context, consistent with their experience and prior knowledge of LLM agents --- is wrong. This is the Eliza effect as a transparency failure: not irrational, but a plausible refinement of a sparse $S_R$ that the system does nothing to correct. The actual structure separates information gathering (p2--p4) from independent asynchronous scoring (p5--p7). The candidate cannot formulate the right questions because their $S_R$ does not contain the structure those questions require.

\paragraph{The Appeals Assessor.} The candidate appeals. The Assessor has broad $S_R$ --- effectively $\mathbf{D}_P^d$ --- and $V_R$ access to the $\mathit{ScoredTemplate}$: questions, answers, and per-question scores. Knowing questions are scored independently, they can evaluate each score against its answer. Suppose q3--q4 contain hostile responses and score zero, while q5--q20 display high competence but score uniformly low. The first is explicable; the second is not. The Assessor can identify the anomaly but cannot resolve it: their $V_R$ does not include which path produced the scores nor the internal scoring logic of p5--p7. The accountability question is unanswerable from their $V_R$ alone --- not because the information does not exist in $\mathbf{D}_P$, but because no channel surfaces it to this role.

Together, the Candidate and Assessor hold complementary and non-overlapping fragments of $\mathbf{D}_P$, each rich where the other is blind. This is a structural property of the system, only visible when roles are modeled explicitly. In general \textit{contestability} of a decision requires not just that some complaint can be given to a decision, but rather that a decision can be examined afterwards. This requires that logs or other records are sufficiently covering of the structure, modelable as a logs role where 
$\langle Q_{logs}, S_{logs}, V_{logs} \rangle$, such that the questions $Q_{logs}$ can be answered given log coverage $S_{logs}, V_{logs}$. 

\subsection{Role-dependent answerability across the scenarios}

Table~\ref{tab:answerability-across-scenarios} presents role-dependent answerability judgments across all five roles identified in Table~\ref{tab:roles}, evaluated pre-integration and under each of the three composition scenarios of Table~\ref{tab:scenarios}. The table assigns to each question A - answerable, P - only partially, and B - blocked, by applying Proposition~\ref{prop:role-answerability}. A slash (e.g.\ A/P) reports structural and operational answerability separately (Definition~\ref{def:answerability-classes}) where the two diverge.
For the $\text{Max}^D_P$ scenario differences are large based on the cohort of candidates which are ranked \textit{top} or the \textit{rest}, which is why these have been separated.

Some roles have information demands which can only be answered by another role with access, such cases of deferral are marked with (A) or (P) respectively.
Questions concerning reviewability or auditability are interpreted over a log view of the composed process, rather than by treating the existence of records as the substantive question itself. This is just like role access is handled. Let ($N_{\log}$) denote the sub-network of ($N$) preserved by the logging process. A question may therefore be answerable from the live process ($N$), or from a role-relative view of ($N$), while being only partially answerable or blocked when evaluated over ($N_{\log}$).
Some higher-level oversight questions, such as R5 and HM5, are derived from clusters of more specific information demands rather than represented by a single inquisitive network. Their verdict reflects the weakest answerability among the information requirements necessary to support the corresponding oversight or responsibility claim.

Several limitations reoccurred broadly across the case-study, among them two were especially limiting and are marked in the table. Firstly the human scoring fallback path, using outsourced human scorers, were not disclosed in any user interface or output to the user, the uncertainty this causes and the consequences of the outsourcing is marked with $\dagger$. Secondly the pre-integration Deployer process involved informal decision-making meetings and deliberations which were not logged beyond personal notes and limited scores. This is marked with $\ddagger$. Where any of these markers appear in the table these problems are the sole reason preventing higher answerability. Beyond this a lack of tooling for working meaningfully with the Provider system beyond the direct scoring pipeline (such as, searches, navigating answers, changing assessments or counterfactual-reruns) were close to the two marked in negative consequence.

\section{Discussion}
\label{sec:discussion}

The case study demonstrates that treating transparency as a relational property is not merely a theoretical reframing -- it is a working analytical instrument. 
Whether a question is answerable, partially answerable, or structurally blocked can be determined from the system's structure, and that determination reveals things that component-level assessment cannot see. Connecting transparency demands to the legitimate questions of stakeholder roles makes it clear that transparency is not just good to have but functional requirement. 

A recurring finding is that the deployment context carries as much analytical weight as the AI system itself. The three integration scenarios share the same Provider system, yet produce fundamentally different distributions of epistemic access, decision weight, and accountability. What varies is not the tool but how it is composed into the deployment environment — and it is that composition, not the tool in isolation, that determines what any role can know, ask, and act on.

A related finding concerns design consequence. The structure of an AI system determines what questions are answerable downstream — what a recruiter can interpret, what a candidate can contest, what an assessor can resolve. These determinations are made before deployment, often without visibility into the role structure the system will affect. Treating transparency as a system property obscures this: it locates the question of transparency inside the tool, where the consequential structural choices have already been made and their effects distributed silently across roles. Treating transparency as a relational property inverts this: the information needs of roles in the deployment environment become design goals, and catering to those needs — making the right questions answerable for the right roles — becomes a concrete and evaluable obligation for the system that enters that environment.

\subsection{Reflections on the Use-Case and Practical Application}
\label{sec:reflections-practical}

The case study illustrates how role-dependent answerability can turn heterogeneous transparency problems into identifiable structural deficiencies. Although the questions in Table~\ref{tab:answerability-across-scenarios} arise from different roles and responsibilities, many limitations can be traced to common features of the composed system. Most notably, the absence of records for informal decisions ($\ddagger$) and the undisclosed human fallback ($\dagger$) affect answerability across several roles and questions. Both problems were subsequently addressed by the Deployer and Provider following the assessment: the Deployer introduced new logging routines for informal decisions, while the Provider was required to disclose which candidates had been scored through the human fallback. A change to a single information-flow relation can therefore improve several otherwise distinct transparency relations.

More generally, the framework provides a practical procedure for examining transparency in a socio-technical system. Stakeholder roles and their legitimate information demands are first identified from the organizational and regulatory context, after which the information flows relevant to answering those questions can be modeled. Role-dependent answerability then makes it possible to locate where required information is available, where access depends on another role, and where it is partial or blocked by the structure of the system. Alternative designs or compositions can consequently be compared in terms of how they change these relations. The framework does not determine which roles or information demands are normatively appropriate, nor does it replace the domain expertise required to identify them; rather, it provides a structural substrate for making their consequences explicit and for identifying where changes to information flow, access, or system functionality can improve transparency.

\subsection{Policy Implications}

Tools intended to operationalize AI governance --- ALTAI, the Dutch FRAIA template, and the Swedish DIGG trust model --- adopt a single-actor perspective, evaluating the AI system or its developer against a checklist of properties. This is understandable since the debate around AI ethics embodied in for instance ALTAI was strongly developer focused and this is reflected in these tools and discussions around them which generally do not mention or support the division of labor and responsibility between different organizational actors \cite{radclyffe2023altai}.   

The framework presented here takes a step toward filling this toolkit-gap. The following considerations and reflections arise from our case-study and from treating transparency as a structural relation between system and role rather than a property of the system alone.

The stakeholder roles connected to a deployment environment carry existing obligations --- organizational, professional, and legal --- that precede any automation requirement. These obligations are the natural source of the legitimate questions $Q_R$ of a role: these should be grounded in what the role is already accountable and responsible for. 
For internal roles such as recruiters and HR managers this includes the obligations the organization operates under regardless of AI involvement, which in the setting of the case study includes the professional requirement to make an informed choice about the candidates, but also legal requirements such as equal treatment regardless of gender, ethnicity or disability and organizational ones such as reporting their decisions in a way that allows for reassessment or audit. 
For external parties --- candidates with a right to understand and contest decisions, auditors with an obligation to verify equal treatment, labor representatives with an interest in process fairness --- the information demands are equally legitimate and equally structural. 

Transparency in this sense forms a precondition for the exercise of rights and fulfillment of obligations before AI enters the picture. 
When AI systems then do enter, the obligations on the Deployer under the AI Act must be understood in light of the obligations the Deployer already acts under --- obligations that the automated component must not make impossible to fulfill. The obligations the Provider must meet follow from this: they are those that enable the Deployer to fulfill its obligations to and for the stakeholders in the deployment environment. 
This strong and explicit association between roles and transparency needs also helps in verifying compliance --- instead of tokenistic disclosure of material the questions shifts to whether the roles in question can do what they need to do or if further information is needed to do so.

That the path and information on which a decision is made determines the weight 
and meaning of that decision is an important consequence  
(Section~\ref{sec:compositional}). This structural property receives too little consideration in assessments of adequacy and responsibility. A common solution to human oversight is Human-in-the-Loop, where a user is put in the position 
to accept or reject automated answers --- as with the recruiters in this case study. 

However, to be effective human oversight obligations require more than a human being present in the decision process. Roles that have access to the right information but no 
mandate or practical possibility to act on it cannot exercise genuine oversight. 
This is compounded by delegated responsibility: AI systems typically interface 
with an immediate user, while accountability for the decision often resides at 
a different organizational level. The framework makes this visible --- a role's 
$S_R$ and $V_R$ can be adequate while their mandate to act on what they know 
is organizationally constrained, leaving responsibility formally assigned but 
practically unexercisable.

A final remark concerns the unit of risk classification. The EU AI Act classifies AI systems as high-risk~\cite{EU2024AIAct}, but the risks, 
obligations, fallback procedures, and accountability chains that make a deployment high-risk reside in the deployment context --- not in the software 
product. The same system can carry fundamentally different risk profiles depending on where it is deployed. 
It is then more reasonable to think of high-risk classification in terms of use contexts and deployment environments on the Deployer side, rather than AI systems on the Provider side. This is not how AI Act obligations are currently structured, but a classification existing on both sides would make procurement, deployment and auditing more straightforward, especially where multiple AI systems are under consideration or use in the same environment.

\subsection{Limitations \& Challenges}

The framework reasons from role definitions -- the questions assigned to $Q_R$, 
the structural knowledge captured in $S_R$, and the operational access 
represented in $V_R$ -- but does not itself determine what these should contain. 
Its conclusions are only as precise as these definitions reflect the actual 
epistemic situation of each role. 

In practice, this is non-trivial: experience from real-world assessments such as the one in this paper, as well as~\cite{brannstrom2025recruitment} suggests that many Deployer organizations struggle 
to articulate their transparency obligations and accountability practices in the presence of AI systems. 
Establishing role definitions that are both formally adequate and grounded in organizational reality remains a practical challenge, also requiring domain expertise and some AI literacy, that this framework by itself does not resolve.

As such, the purpose and role of this formal framework is to provide grounding for a modeling substrate which \emph{allows} a toolkit that does address these practical problems to model Deployer environments, AI systems and their combinations. In fact, while this framework in part grew out of the limitations of existing assessment tools to meet the demands of real-world use cases, the purpose is not to replace such tools but to augment them with a different object of study.

\subsection{Further Directions}

Role-dependent answerability opens directly onto \textit{effective} transparency: not merely whether a question is answerable, but whether the answer affords the action or judgment the role requires~\cite{ananny2018seeing}.
A recruiter who can technically access a scored template but lacks the structural knowledge to interpret it is not effectively transparent to the system. Effective transparency also requires information to be understandable and simulatable by the receiver~\cite{paez2025axe, lipton2018mythos}. A complementary, grasp-based account would concern which questions remain open, expressible, or cognitively available to an inquirer --- some may be answerable in our sense but not graspable, others structurally blocked, and some not yet posable because the relevant distinctions are unavailable to the inquirer. Formalizing the connection between answerability and agency, including inquirer-side understanding, is the most significant theoretical direction of extension, with direct implications for how transparency obligations in AI governance might be specified and evaluated.

\section{Conclusion}

Transparency is frequently mandated in AI governance and rarely specified. This paper has developed a formal account of transparency as answerability --- a relational property between a system and the questions a specific role needs answered --- and illustrated it on a real pre-integration assessment. The framework characterises transparency along three dimensions: structural answerability within a system, compositional answerability across organisational and technical boundaries for deployed systems, and role-dependent answerability given a stakeholder's accessible sub-network. Information flow networks provide the structural substrate: a common formalism in which technical components, human processes, organisational roles, and how they are composed together, can be represented, queried, and compared.

The most significant open direction is effective transparency: answerability that affords action. Making questions answerable is necessary but not sufficient --- what matters ultimately is whether the answer enables the role to do what the transparency is for. That is where the formal and normative work converges, and where the framework points next.

\begin{acks}
This work was supported by the ``AEQUITAS'' project funded by the European Union's
Horizon Europe research and innovation programme under grant number 101070363.
The views and opinions expressed are those of the author(s) only and do not
necessarily reflect those of the European Union or the granting authority.

The authors thank the auditing firm that conducted the AI system assessment and the
participating municipality for their collaboration and engagement in the case study.
Conflict of interest: author Mattias Br\"annstr\"om is a co-owner of the auditing firm;
the firm was compensated for its consultancy work on the assessment, but not for the
academic research and analysis reported in this paper.
\end{acks}

\bibliographystyle{ACM-Reference-Format}
\bibliography{ref}

\end{document}